\documentclass[conference]{IEEEtran}
\IEEEoverridecommandlockouts
\usepackage{amsmath}
\usepackage{amssymb}
\usepackage{mathtools}
\usepackage{amsthm}
\usepackage{booktabs}
\usepackage{cite}
\usepackage{hyperref}
\usepackage[capitalize,noabbrev]{cleveref}

\usepackage{dsfont}
\usepackage{graphicx}
\usepackage{textcomp}
\usepackage[table]{xcolor}
\usepackage{colortbl}
\usepackage{xspace}
\usepackage{tabularx}
\usepackage{multirow}
\usepackage{subcaption}
\usepackage[inline]{enumitem}
\usepackage{soul}
\usepackage{xurl}
\usepackage{wrapfig}
\usepackage{float}
\usepackage{amsthm}
\usepackage{color}
\usepackage{algorithm}
\usepackage{algorithmic}
\usepackage{pgfplots}
\pgfplotsset{compat=1.18}

\theoremstyle{plain}
\newtheorem{theorem}{Theorem}[section]
\newtheorem{prop}[theorem]{Proposition}
\newtheorem{lemma}[theorem]{Lemma}

\theoremstyle{definition}
\newtheorem{definition}[theorem]{Definition}

\theoremstyle{remark}
\newtheorem{remark}[theorem]{Remark}
\newtheorem{claim}{Claim}

\newcommand{\xx}{\boldsymbol{x}\xspace}
\newcommand{\adv}{\hat{\boldsymbol{x}\xspace}}
\newcommand{\lcadv}{\Tilde{\hat{\boldsymbol{x}\xspace}}}
\newcommand{\lczz}{\Tilde{{\boldsymbol{z}\xspace}}}
\newcommand{\weight}{\boldsymbol{\theta}\xspace}

\newcommand{\detector}{\mathbb{D}\xspace}
\newcommand{\lcx}{\Tilde{\boldsymbol{x}}\xspace}

\newif\ifshowrevisions
\showrevisionsfalse % Clean version: hide revision margin tags.
\colorlet{blue}{black} % Clean version: render revision-blue text in black.
\newif\ifarxivversion
\newif\ifcameraready
\newif\ifshowmainalgorithms
\arxivversionfalse
\camerareadyfalse
\showmainalgorithmsfalse
\newcommand{\prooflocation}[1]{\ifcameraready the full version\else Appendix~\ref{#1}\fi}
\definecolor{revisiontagcolor}{RGB}{110,70,150}
\newcommand{\tagsep}{,\allowbreak\hspace{0pt}}
\newcommand{\margintext}[1]{%
  \parbox[t]{\marginparwidth}{\centering\scriptsize\textcolor{revisiontagcolor}{#1}}%
}

\newcommand{\revise}[2]{%
  \ifshowrevisions
    \marginpar{\margintext{#1}}\textcolor{blue}{#2}%
  \else
    #2%
  \fi
}
\newcommand{\revisetag}[1]{\ifshowrevisions\marginpar{\margintext{#1}}\fi}
\newcommand{\advadp}{\hat{\boldsymbol{x}}_{\mathrm{adp}}\xspace}

\definecolor{mynicegreen}{RGB}{171,195,47}
\definecolor{mynicered}{RGB}{255,154,154}
\definecolor{myniceblue}{RGB}{153,153,255}

\arxivversiontrue
\camerareadyfalse
\showmainalgorithmstrue

\makeatletter\def\@IEEEpubidpullup{6.5\baselineskip}\makeatother
\IEEEpubid{\parbox{\columnwidth}{
        Network and Distributed System Security (NDSS) Symposium 2027\\
        22–26 March 2027, Seoul, Republic of Korea\\
        ISBN 978-1-970672-09-1\\
        https://dx.doi.org/10.14722/ndss.2026.230728\\
        www.ndss-symposium.org
}
\hspace{\columnsep}\makebox[\columnwidth]{}}

\begin{document}

\title{Adversarial Detection for Pre-trained Classifiers via Lipschitz-Margin Analysis with Certification}

\author{
\IEEEauthorblockN{
Rui Zhang\IEEEauthorrefmark{1}\textsuperscript{,}\IEEEauthorrefmark{5}\thanks{\IEEEauthorrefmark{5}This work was started while Rui Zhang was at Zhejiang University and a visiting student at the University of Waterloo under the supervision of Prof.Qingbiao Wu, Prof. Jian Liu, and Prof. N.Asokan. The technique report version is at https://arxiv.org/abs/2610.15359},
Youwen Zhu\IEEEauthorrefmark{1}\textsuperscript{,}\IEEEauthorrefmark{4},
Xinle Cao\IEEEauthorrefmark{2},
Huazhen Yu\IEEEauthorrefmark{3},
Jian Liu\IEEEauthorrefmark{2}\textsuperscript{,}\IEEEauthorrefmark{4}
}
\IEEEauthorblockA{\IEEEauthorrefmark{1}Nanjing University of Aeronautics and Astronautics}
\IEEEauthorblockA{\IEEEauthorrefmark{2}The State Key Laboratory of Blockchain and Data Security, Zhejiang University}
\IEEEauthorblockA{\IEEEauthorrefmark{3}Nanjing University of Information Science and Technology \quad \IEEEauthorrefmark{4}Corresponding author}
}

\maketitle

\begin{abstract}
Deep classifiers deployed in safety-critical services like autonomous-driving perception remain vulnerable to small adversarial perturbations. Certified $\ell_2$ defenses today face a costly trade-off: randomized smoothing (RS) requires hundreds of noisy queries per input, while Lipschitz networks (LN) constrain the architecture and sacrifice clean accuracy. We take a detection-oriented approach instead: rather than certifying that predictions are invariant inside an $\ell_2$ ball, we certify that adversarial inputs \revisetag{Comment-1}\textcolor{blue}{within an $\ell_2$ ball} are \emph{detected} --- a guarantee aligned with security-critical deployments where missing an attack is costlier than flagging a clean input.

To realize this, we prove a Lipschitz-margin reduction from certified $\ell_2$ robustness to certified detection, which reduces the problem to handling a single Lipschitz constant $L$. We introduce two detector families: $\detector_{Adv}$, which estimates $L$ from a single defender-side perturbed pair, and $\detector_{lc}$, which augments that pair with a local margin centroid to keep the estimate above what the certificate requires. Both deploy on any pretrained classifier without retraining. \textcolor{blue}{ \revisetag{\mbox{Review-A-1}\tagsep\mbox{Comment-1}}With a valid global bound, the detector is fully certified at that radius; with a local estimate $L'$, certification is sample-wise and applies only when our certified condition holds.}

Across MNIST, CIFAR-10, and ImageNet-1k on four backbones (LeNet-5, VGG16, ResNet-50, ViT-B/16), our detectors match or outperform five baselines (CTT, SortNet, CRM, ContraNet, A\textsuperscript{2}D). \textcolor{blue}{For CIFAR-10 with 1-Lipschitz SortNet, our reduction delivers $100\%$ certified detection at the lowest FPR among compared certified detectors; on the standard CIFAR-10 VGG16, our detectors reject all evaluated in-budget adversarial examples, with 62-82\% both rejected and formally certified.
}
\end{abstract}

\IEEEpeerreviewmaketitle

\section{Introduction}
\label{sec:introduction}

Deep classifiers (DNNs) deployed in safety-critical services like autonomous-driving perception~\cite{eykholt2018robust}, content moderation, biometric authentication~\cite{sharif2016accessorize}, and access control are vulnerable to \emph{adversarial examples}~\cite{intri2013,jing2021too}: small, often imperceptible perturbations $\delta$ that flip the prediction $\arg\max f(x+\delta) \neq \arg\max f(x)$.

A large body of work~\cite{li_sok_2023,weng2020towards,lecuyer2019certified,pautov2022cc,li2021towards} addresses this by building \emph{certifiably robust} classifiers --- whose prediction is provably unchanged inside an $\ell_p$-ball around the input. Existing $\ell_2$ certifications cluster into two families, each with a structural cost:

\begin{itemize}[leftmargin=*, itemsep=0pt, topsep=0pt, parsep=0pt, partopsep=0pt]
    \item \textbf{Randomized smoothing (RS)}~\cite{cohen2019certified,salman2019provably,yang2020randomized,li_certified_2019,carlini_certified_2023,delattre2024lipschitz} replaces $f$ with a smoothed classifier whose output is a majority vote over Gaussian-noised inputs. Each prediction needs hundreds to thousands of forward passes, ruling out real-time use.
    \item \textbf{Lipschitz networks (LN)}~\cite{zhang_rethinking_2022,leino2021globally,li_preventing_2019} constrain the architecture so that a small global Lipschitz constant is provable by construction. The constraint sacrifices clean accuracy and forecloses use of standard pretrained classifiers.
\end{itemize}

An alternative, motivated by Shumailov et al.~\cite{shumailov2020towards} and Tramer~\cite{tramer2022detecting}, is to defend by \emph{detecting and rejecting} adversarial inputs: the deployed classifier is left unchanged, and a detector decides whether each query is safe to answer. This preserves clean accuracy and inference speed. Empirical detectors are well studied~\cite{gao2026from, ma2019nic, aldahdooh_adversarial_2022, yang2022you, sotgiu2020deep, zhao2024attack, lai2025auditvotes}, but they offer no formal guarantee, and \emph{certified detection} has received far less attention. Tramer~\cite{tramer2022detecting} gave a theoretical reduction from robust classification to robust detection, but the construction is too expensive for deployment. We therefore ask:
\textit{Can a certified robust classifier be turned into a certified detector that is efficient enough for real-time use?}

\textbf{Three challenges.} Doing this in practice requires solving three coupled problems.

\emph{(C1)} \textit{Reduction.} Existing certified-robustness theorems guarantee \emph{prediction invariance} inside a radius, not detectability of adversarial inputs outside it. Reducing detection to robustness requires a fresh argument that ties the margin of any in-budget adversarial example to its perturbation budget.

\emph{(C2)} \textit{Lipschitz constants are intractable.} Any Lipschitz-based reduction needs a value for the network's Lipschitz constant $L$. Architecturally enforced bounds (e.g., 1-Lipschitz networks) cost clean accuracy; numerical bounds on standard pretrained networks are vacuous.

\emph{(C3)} \textit{Local estimators must be calibrated.} Replacing the unknown global $L$ with a local estimator $L'$ can occasionally yield $L'$ below what the certificate requires, silently breaking the guarantee on those samples.

\textbf{Our approach.} An adversarial example with bounded budget cannot move far from the decision boundary, so the classifier's margin function $q(\cdot)$ must be small there. Lipschitz continuity makes this quantitative: if $f$ is $L$-Lipschitz, the change in margin under any perturbation $\delta$ is bounded by $cL\|\delta\|$, tying the adversarial margin $q(\adv)$ to the clean margin $q(\xx)$ and to $\|\delta\|$. Each input is certified by a Lipschitz value relevant \emph{to it}, not a globally enforced one. Two practical detectors follow: $\detector_{Adv}$ replaces the unknown clean counterpart with a defender-side perturbed surrogate, and $\detector_{lc}$ additionally anchors the bound to a local centroid of the margin function. The two detectors are tighter in different regimes (cf.\ Section~\ref{sec:detector_comparison}): on Lipschitz-constrained networks $\detector_{Adv}$'s short-pair certificate is naturally tight, while on unconstrained pretrained classifiers $\detector_{lc}$'s centroid scales the bound with the model's clean-margin distribution and certifies a substantially larger fraction. Both deploy on any pretrained classifier.

For \emph{(C1)}, we prove a Lipschitz-margin reduction (Lemma~\ref{lemma.certified_detection}, Theorem~\ref{theorem.lowerbound_of_distance}): every successful in-budget adversarial example is provably detectable from $L$ and the margin alone. The certificate is \emph{attack-agnostic}: it applies to any in-budget $\ell_2$ perturbation that flips the prediction, regardless of how it was generated. For \emph{(C2)}, on Lipschitz-constrained networks (SortNet~\cite{zhang_rethinking_2022}, CRM~\cite{fazlyab2023certified}) the exact $L_{\text{glob}}$ gives a fully certified detector via Lemma~\ref{lemma.certified_detection}; \textcolor{blue}{on unconstrained pretrained classifiers,\revisetag{Comment-1} we replace $L_{\text{glob}}$ with a local estimate $L'$, which yields sample-wise certificates when the local estimate is large enough for certification while allowing substantially lower FPR in practice. For \emph{(C3)}, we quantify this sample-wise guarantee with \textbf{cert\%}, the fraction of successful in-budget adversarial examples that are both rejected and formally certified. Concretely, for each detected adversarial example, the certificate holds when $L' \geq L^*$, where $L^*$ is the minimum Lipschitz value required by Proposition~\ref{prop:pairwise_suffices} on that sample. To increase this certified subset, we can use a more conservative multi-probe estimate $L'_{\mathrm{strong}}=\max_k L'_k$: it deliberately raises $L'$, making $L'\geq L^*$ easier to satisfy, but the same larger $L'$ also enlarges the rejection threshold, increasing FPR and runtime.}

\textcolor{blue}{\textbf{What our guarantee covers.} We distinguish three protection levels.\revisetag{Comment-1\tagsep\mbox{Review-A-1}\tagsep\mbox{Review-A-2}\tagsep\mbox{Review-C-4}} \emph{Fully certified detection} applies when a valid Lipschitz bound is known for all inputs, so every successful adversarial example within the declared radius is rejected; the cost is that a conservative global bound can also reject many clean inputs, i.e., high FPR. \emph{Conditionally certified detection} applies on standard pretrained models after replacing the intractable global bound by a local estimate $L'$: this gives a lower-FPR operating point, but only the subset satisfying the per-sample certified condition is formally certified. The stronger $L'$ variant above deliberately moves along this frontier, trading FPR and runtime for a larger certified subset. \emph{Empirical protection} describes additional successful attacks that are rejected in experiments but do not satisfy this per-sample certificate. We therefore report TPR, FPR, and cert\% together: the same detector may reject nearly all measured attacks, while only a subset of those rejections are formally certified.}

In summary, our contributions are as follows.
\begin{itemize}[leftmargin=*, itemsep=2pt, topsep=0pt, parsep=0pt, partopsep=0pt]
\item A \textbf{Lipschitz-margin reduction} from certified $\ell_2$ robustness to certified adversarial detection (Lemma~\ref{lemma.certified_detection}, Theorem~\ref{theorem.lowerbound_of_distance}) with an attack-agnostic guarantee. We empirically validate the reduction on 1-Lipschitz SortNet, where the global Lipschitz $L_{\text{glob}}$ is exactly known and Lemma~\ref{lemma.certified_detection} delivers fully certified detection (Table~\ref{tab:sortnet_certified}).

\item Two \textbf{practical approximated detectors} for the case where $L_{\text{glob}}$ is intractable: $\detector_{Adv}$ (single-perturbation surrogate, tighter on Lipschitz-constrained networks) and $\detector_{lc}$ (local-centroid surrogate, tighter on unconstrained pretrained classifiers); both deploy without retraining. We give deterministic bias bounds (Section~\ref{sec:method}) and introduce \textbf{cert\%}, \textcolor{blue}{the fraction of all successful in-budget adversarial examples that are both rejected and formally certified.  \revisetag{Comment-1} On CIFAR-10 VGG16, our detectors achieve $100\%$ empirical TPR on in-budget adversarial examples while $62$--$82\%$ of them are both rejected and formally certified; a stronger multi-probe $L'$ estimate further raises CIFAR-10 cert\% by up to $34.8$ points with explicit FPR and runtime costs.\revisetag{Review-A-3}} (Tables~\ref{tab:cifar_compact} and~\ref{tab:strong_estimator}).

\item A \textbf{comprehensive empirical evaluation} on MNIST, CIFAR-10 and ImageNet-1k, four architectures (LeNet-5, VGG16, ResNet-50, ViT-B/16), and five baselines (CTT, SortNet, CRM, ContraNet, A\textsuperscript{2}D). Against SOTA empirical detector A\textsuperscript{2}D we deliver $100\%$ TPR at $3.6\%$ FPR on ViT-B/16 ($\epsilon{=}0.1$) versus A\textsuperscript{2}D's $89.5\%$ TPR at $11.2\%$ FPR. Furthermore, our detection runs $9$--$20\times$ faster than randomized smoothing on ImageNet (Section~\ref{sec:experiments}).

\item A \textbf{principled bridge from robust training to detection.} A large body of robustness research focuses on enlarging clean-input margins (Section~\ref{related}). Because our detectors deploy on any pretrained classifier without retraining, margin-enlarging advances in that literature can tighten our detection frontier under the same reduction, e.g., $\ell_2$ PGD adversarial training on a CIFAR-10 VGG16 cuts $\detector_{lc}$'s FPR from $93.6\%$ to $32.2\%$ at $\epsilon{=}0.5$ while maintaining $\geq 99\%$ TPR (Section~\ref{sec:robust_boosts}).
\end{itemize}

\section{Related Works\label{related}}

\subsection{Empirical Adversarial Detectors}

A long line of empirical detectors aims to identify adversarial inputs at inference time without certified guarantees~\cite{aldahdooh2022adversarial,zhao2023adversarial,gao2026from,ma2019nic}. They roughly fall into three categories. \emph{Auxiliary-network detectors} attach trainable components to the protected classifier and flag inputs whose intermediate or reconstructed representations look anomalous: ContraNet~\cite{yang2022you} trains a class-conditional generator and uses reconstruction error as the detection signal but does not scale to ImageNet-1k, while U-CAN~\cite{mizrahi2025pulling} embeds projection layers and ArcFace-style components into intermediate features and detects adversarial inputs through unsupervised anomaly detection. \emph{Statistic-based detectors} compute test statistics from the model's predictions or gradients on the input; PASA~\cite{bhusal2024pasa} is a recent representative that combines prediction sensitivity and feature-attribution sensitivity to noise into two threshold-calibrated scores. Finally, \emph{adversarial attacks as detectors} flip the perspective: A\textsuperscript{2}D~\cite{zhao2024attack} measures the perturbation budget required to cross the decision boundary, exploiting the fact that already-adversarial inputs need less push than clean ones, and is one of the few empirical detectors evaluated on full ImageNet-1k --- making it our strongest empirical baseline (Table~\ref{tab:vit_a2d}).

All three families provide \emph{empirical} guarantees only and remain subject to the standard adaptive-attack critique~\cite{tramer2020adaptive,wan2023bounceattack}: an adversary aware of the detector can craft inputs that evade its statistics or auxiliary networks while still flipping the underlying classifier. In contrast, our detectors derive from a Lipschitz-margin reduction (Section~\ref{sec:main_results}) and inherit an \emph{attack-agnostic} certificate: any in-budget $\ell_2$ perturbation that flips the prediction is provably refused, regardless of how the adversarial input was generated.

\subsection{Certified Robustness}
Existing work on certified robustness against adversarial examples primarily focuses on designing robust classifiers that maintain correct predictions even when inputs are adversarially perturbed. These approaches can be broadly categorized into three families: (1) relaxation-based methods~\cite{gowal2019scalable}, (2) Lipschitz network methods~\cite{leino2021globally}, and (3) randomized smoothing methods~\cite{cohen2019certified}. Among them, randomized smoothing and Lipschitz network methods currently achieve state-of-the-art performance against $\ell_2$-norm perturbations; consequently, we concentrate on these two paradigms.

Randomized smoothing (RS) methods~\cite{cohen2019certified,li_certified_2019,wong_scaling_2018,carlini_certified_2023,mueller2023certified} enhance robustness by injecting random noise into the input, thereby transforming an arbitrary base classifier into a \emph{smoothed classifier} that is certifiably robust under the $\ell_2$ norm. In particular, Cohen et al.~\cite{cohen2019certified} introduced Gaussian smoothing and derived the strongest known $\ell_2$ robustness radius. For $\ell_1$ robustness, Lecuyer et al.~\cite{lecuyer2019certified} employed Laplacian smoothing. Regarding $\ell_\infty$ robustness, Yang et al.~\cite{yang2020randomized} demonstrated that by leveraging norm inequalities to translate $\ell_2$ guarantees derived from Gaussian smoothing into $\ell_p$ bounds, one can achieve state-of-the-art certified $\ell_\infty$ robustness.

For Lipschitz network (LN) methods, they exploit the Lipschitz constant of the neural network to formally certify robustness~\cite{xu_lot_2023,zhang_rethinking_2022,hein_formal_2017,li_preventing_2019,leino2021globally,tsuzuku2018lipschitz}. Early efforts in this line typically compute a global Lipschitz constant~\cite{leino2021globally,tsuzuku2018lipschitz} and introduce novel architectures or training procedures to tighten this bound. More recent works directly enforce tight Lipschitz constraints by designing 1-Lipschitz networks~\cite{zhang_rethinking_2022,xu_lot_2023}. Fazlyab et al.~\cite{fazlyab2023certified} (CRM) maximises a dynamic margin while regularising the Lipschitz constant; we use both SortNet~\cite{zhang_rethinking_2022} and CRM as certified-detection baselines (Section~\ref{sec:experiments}). Notably, Delattre et al.~\cite{delattre2024lipschitz} recently proposed a hybrid approach that integrates Lipschitz networks with randomized smoothing.

\subsection{Margin and Adversarial Robustness}
A complementary line of work analyses adversarial robustness through the \emph{margin} of the classifier. Tsuzuku et al.~\cite{tsuzuku2018lipschitz} formalised the link --- a classifier with margin $q(\xx)$ at point $\xx$ is provably robust to $\ell_2$ perturbations of size up to $q(\xx)/L$, motivating margin-maximising training. Xu et al.~\cite{xu2023boundary} empirically showed that classifiers' adversarial robustness is governed by decision-boundary dynamics, with larger clean-input margins translating to harder-to-attack inputs. Lyu et al.~\cite{lyu_gradient_2020,lyu_gradient_2021} prove that gradient descent on homogeneous neural networks converges to a KKT point of the margin-maximising problem, a result we leverage to compute the local centroid in $\detector_{lc}$ (Section~\ref{sec:method}). The same margin lens motivates adversarial training~\cite{PGD,xu2020adversarial,yan_adversarial_2021,shafahi2019adversarial} as a margin-enlarging primitive on the clean-input side.

Our work shares the margin-centric view but reaches it from the opposite direction. Rather than asking \emph{how to push clean-input margins outward}, we ask \emph{where an adversarial example must land given a bounded budget}: Lipschitz continuity forces it to remain close to the decision boundary in margin space, which makes the margin function an attack-agnostic detection signal. The two perspectives converge --- both single out the margin as the discriminator between robust and non-robust regimes --- but ours yields a quantitative reduction from certified robustness to certified detection (Section~\ref{sec:main_results}) rather than a robustness-training prescription. As a consequence, every margin-enlarging primitive in this literature --- adversarial training, Lipschitz constraints, Gaussian augmentation --- composes with our framework: it lifts the underlying classifier's clean-margin distribution and directly tightens our certified-detection frontier (Section~\ref{sec:robust_boosts}).

\subsection{Certified Detection}
Shumailov et al.~\cite{shumailov2020towards} (CTT) introduced an early certified detection framework with a one-sided guarantee on the \emph{specificity} side: samples flagged as adversarial are certified to be adversarial, with non-zero false-negative rate on real adversarial inputs. CrossCert~\cite{zhou2024crosscert} (FSE 2024) extends one-sided certified detection to patch attacks via a cross-checking mechanism. Our work parallels this line for $\ell_2$-bounded perturbations on the complementary \emph{recall} side: every in-budget adversarial example is provably refused, with non-zero FPR on clean inputs (cf.\ Section~\ref{sec:detection_problem}). The two families thus offer complementary safety properties, and the right choice in deployment depends on whether false-positives or false-negatives are more costly.

\section{Problem Formulation}\label{sec:problem_formulation}

In this section we formalize the certified adversarial detection problem.
We describe the threat model (\S\ref{sec:threat_model}), introduce the notation and preliminaries used throughout the paper (\S\ref{sec:prelim}), and then state the formal detection problem and the certification objective (\S\ref{sec:detection_problem}).

\subsection{Threat Model}\label{sec:threat_model}

\paragraph{Deployment scenarios.}
We consider deployments where a deep classifier $f: \mathbb{R}^d \rightarrow \mathbb{R}^m$ is exposed to untrusted inputs from users or sensors:
\begin{itemize}[leftmargin=*, itemsep=0pt, topsep=0pt, parsep=0pt, partopsep=0pt]
    \item \textbf{Autonomous-driving perception.} A traffic-sign or pedestrian classifier ingests camera frames; a misclassification under a small physical-world perturbation (e.g., a sticker on a stop sign) can drive unsafe trajectory decisions and cause physical harm.
    \item \textbf{Content moderation.} A platform-side image classifier flags uploads that violate policy; an adversary perturbs a violating image just enough to evade flagging while remaining recognisably the same content.
    \item \textbf{Biometric authentication \& access control.} A face- or voice-recognition gate decides whether to grant access; an adversary perturbs an unauthorised face to be classified as an authorised identity.
\end{itemize}
In all three cases $f$ and its parameters are \emph{fixed at deployment}: retraining or modifying $f$ is not available to the defender, and inference must meet a real-time latency budget (tens of milliseconds for AV perception, hundreds for moderation/auth).

\paragraph{Attacker goal and capability.}
Given a clean input $\xx$ that $f$ classifies correctly, the attacker seeks a perturbation $\delta$ that flips the prediction, $\arg\max f(\xx + \delta) \neq \arg\max f(\xx)$, while keeping $\delta$ small in the $\ell_2$ norm: $\|\delta\|_2 \leq \epsilon$.
The budget $\epsilon$ is \revise{Review-B-2}{chosen by the defender and announced as a deployment-side protection level:\revisetag{Comment-1} within this declared $\ell_2$ radius, the detector can reject adversarial inputs, with high TPR and formal guarantees for the certified subset.}
We primarily consider $\ell_2$; extension to other $\ell_p$ norms is discussed in Section~\ref{sec:lp_extension}.

\paragraph{Attacker knowledge and adaptivity.}
We assume a white-box attacker with full access to $f$'s architecture and parameters, consistent with the standard threat model in the adversarial ML literature~\cite{carlini2017towards, PGD}.
\textcolor{blue}{The defender likewise has white-box access to $f$, \revisetag{Review-B-3} which is natural in model-owner deployments where the defender controls the classifier and inserts the detector, although the service itself may remain black-box to external users.}
The attacker may use any attack procedure (FGSM, BIM, PGD, C\&W, and adaptive variants aware of the detector).
Crucially, our certified guarantee is \emph{attack-agnostic}: by construction it holds for every $\adv = \xx + \delta$ with $\|\delta\|_2 \leq \epsilon$, independent of how $\delta$ was generated.
An adaptive attacker therefore cannot bypass the certificate by switching optimisation strategy --- any in-budget perturbation that flips the prediction is provably detected (Section~\ref{sec:adaptive_attack}).

\paragraph{Defender architecture: detection-then-refuse.}
The defender deploys $f$ together with a detector $\detector: \mathbb{R}^d \rightarrow \{0, 1\}$.
On input $z$, the pipeline returns $f(z)$ when $\detector(z) = 0$ and \emph{refuses} the input when $\detector(z) = 1$.
Refusal is a hard reject signal: it does not return any class label and does not attempt to recover the original clean input.
Refused inputs are escalated to a deployment-defined fallback --- human review for content moderation, an additional verification factor for biometric authentication, or a safe-stop policy for AV perception.
This \emph{detection-then-refuse} paradigm contrasts with \emph{correction}-based defenses (which try to recover the clean label from a perturbed input) and with \emph{soft-rejection} defenses (which lower the output confidence without rejecting): we make no recovery attempt and emit a binary refuse signal.

\paragraph{Defender objective.}
\textcolor{blue}{The defender's primary objective is to refuse every successful in-budget adversarial example.\revisetag{Comment-1} This objective is fully certified when a valid global Lipschitz bound is available. For standard pretrained models, where $L_{\text{glob}}$ is unavailable, our detectors approximate this objective: they reject successful in-budget adversarial examples empirically, and we report \textbf{cert\%} for the subset whose rejection is formally certified.}
Refusing some clean inputs (non-zero FPR) is a willing trade-off, calibrated by deployment context: an AV perception pipeline can absorb conservative refusal by deferring to the safe-stop fallback, content moderation can route flagged uploads to human reviewers, and biometric gates can request a second factor. A secondary objective is inference efficiency --- the per-input overhead of $\detector$ must remain within the latency budget of the host pipeline.
We make no claim about robustness to attackers operating outside the announced $\ell_2$ budget; that is the responsibility of the announced safety envelope.

\subsection{Notation and Preliminaries}\label{sec:prelim}

A complete summary of symbols used throughout the paper is given in Appendix~\ref{appendix:symbols} (Table~\ref{tab:notation}).

We denote a clean input by $\xx$ and its adversarial counterpart by $\adv = \xx + \delta$, where $\epsilon$ is the upper bound of perturbation.
We further use $z$ for a generic test input at inference time (clean or adversarial --- the defender does not know which) and $z' = z + \sigma$ for a deliberate, defender-side perturbation of $z$ used as a numerical surrogate for $z$'s clean counterpart; the attacker-side $\adv$ and the defender-side $z'$ live on opposite sides of the inference pipeline.
Let $f:\mathbb{R}^d \rightarrow \mathbb{R}^m$ be a neural network, with $f_i(\xx)$ the $i^{\text{th}}$ output component (the confidence that $\xx$ belongs to class $i$), and let $\weight$ denote the parameters of $f$.
The $\ell_p$ norm of a vector is $\|\xx\|_p = \big(\sum_{i}|x_i|^p\big)^{1/p}$; unless otherwise stated we default to the $\ell_2$ norm.

The theoretical foundation of our method relies on the Lipschitz continuity of $f$. Intuitively, Lipschitz continuity bounds how much the output changes under a given input perturbation.

\paragraph{Lipschitz continuity.}
A mapping $f:\mathbb{R}^d \rightarrow \mathbb{R}^m$ is $L$-Lipschitz continuous with respect to norm $\|\cdot\|$ if, for all $\xx_1, \xx_2 \in \mathbb{R}^d$,
\begin{equation}
     \|f(\xx_1)-f(\xx_2)\|\leq L\|\xx_1-\xx_2\|.
\end{equation}

\paragraph{Margin.}
For a reference class $k$, the \emph{class-$k$ margin} of an input $u$ with respect to $f$ is
\begin{equation}\label{eq:q-fixed-label}
    q(u; k) \;=\; f_k(u) - \max_{j\neq k} f_j(u).
\end{equation}
We write $q(u)$ as shorthand for $q(u; \arg\max_k f_k(u))$, i.e., the margin at $u$'s \emph{own predicted class}; this is non-negative by construction. When the reference class is fixed and clear from context (e.g., the pseudo-label $\hat y = \arg\max f(z)$ of an input $z$ throughout a detection step), we keep the shorthand $q(\cdot)$ but note that the value can become \emph{negative} if the second argument differs from the reference class --- as happens to $q(z'; \hat y)$ when the defender-side perturbation $z'$ flips $z$'s prediction (Algorithm~\ref{alg:perturbation}).
For Lemma~\ref{lemma.certified_detection}, both $q(\xx)$ and $q(\adv)$ are evaluated at their own predicted classes (so both are non-negative and the bound is symmetric).

\paragraph{Certified robustness.}
For a neural network $f$ with Lipschitz constant $L$ under $\ell_p$, the model is provably robust to perturbations $\|\delta\|_p < \frac{1}{cL} q(\xx)$, i.e.,
\begin{equation}
\arg\max f(\xx) = \arg\max f(\xx+\delta), \quad \forall \,\|\delta\|_p<\frac{1}{cL}q(\xx),
\end{equation}
where $c=2^{\frac{p-1}{p}}$ depends on the norm.
This is the standard robustness certificate~\cite{li_sok_2023, weng2020towards, tsuzuku2018lipschitz}; our work reduces \emph{detection} to it.

\subsection{Certified Adversarial Detection}\label{sec:detection_problem}

We now formalize the detection problem and the certification objective.

\begin{definition}[Certified Detector]\label{def:certified_detector}
Let $f: \mathbb{R}^d \rightarrow \mathbb{R}^m$ be a classifier and let $\epsilon > 0$ be a perturbation budget.
A detector $\detector: \mathbb{R}^d \rightarrow \{0, 1\}$ is a \emph{certified detector at budget $\epsilon$} (under $\ell_p$) if, for every clean input $\xx$ and every adversarial example $\adv = \xx + \delta$ satisfying $\|\delta\|_p \leq \epsilon$ and $\arg\max f(\adv) \neq \arg\max f(\xx)$, we have $\detector(\adv) = 1$.
\end{definition}

\begin{remark}[One-sided guarantee and non-triviality]\label{rem:onesided}
Definition~\ref{def:certified_detector} is a \emph{one-sided} certificate: $\detector$ must reject every in-budget successful adversarial example, but may also reject some clean inputs (non-zero FPR).
This parallels the complementary guarantee of CTT~\cite{shumailov2020towards}, which certifies that clean samples are \emph{not} flagged and admits non-zero FNR on adversarial inputs, and of CrossCert~\cite{zhou2024crosscert}, which defines the same form of one-sided certified detection against patch attacks.
It also mirrors certified robustness, which guarantees correct classification inside a radius while admitting non-trivial error outside it.
Crucially, joint reporting of TPR and FPR in our evaluation rules out trivial detectors: the constant detector $\detector \equiv 1$ satisfies Definition~\ref{def:certified_detector} but has FPR $= 1$.
\end{remark}

\paragraph{Evaluation metrics.}
Given a test set of clean inputs and successful adversarial examples, we report three quantities:
\begin{itemize}[leftmargin=*, itemsep=0pt, topsep=0pt, parsep=0pt, partopsep=0pt]
    \item \textbf{TPR}: the fraction of \emph{successful} adversarial examples (those that change $f$'s prediction) that $\detector$ rejects.
    \item \textbf{FPR}: the fraction of clean samples that $\detector$ rejects.
    \item \textbf{cert\%}: \textcolor{blue}{the fraction of all successful in-budget adversarial examples that are both rejected and formally certified, i.e.,\revisetag{Comment-1} the local Lipschitz estimate $L'$ used by $\detector$ satisfies the per-detector certification condition on that pair (cf.\ Section~\ref{sec:approx_to_cert})}.
\end{itemize}

\paragraph{Design objective.}
Given a pretrained classifier $f$ and a budget $\epsilon$, we seek a certified detector (Definition~\ref{def:certified_detector}) that additionally
(i) achieves low FPR at $\epsilon$ of practical interest,
(ii) is applicable to any pretrained $f$ without retraining or architectural constraints, and
(iii) has inference-time overhead within the latency budget of real-time pipelines such as AV perception.
Section~\ref{sec:main_results} develops the Lipschitz-margin analysis that realizes this objective.

\section{From Certified Robustness to Certified Detection~\label{sec:main_results}}

We now formalize the Lipschitz-margin reduction sketched in Figure~\ref{fig:overview} and instantiate two approximated detectors that operate on any pretrained classifier without retraining. The load-bearing results are Lemma~\ref{lemma.certified_detection} (the core inequality) and Theorem~\ref{theorem.lowerbound_of_distance} (its centroid-distance form); Claims~\ref{claim:approximate_adv}--\ref{claim:approximate_lc} instantiate the two approximated detectors; the bias bounds (Section~\ref{sec:bias analysis}) and certification conditions (Section~\ref{sec:approx_to_cert}) tighten the connection between them.

\subsection{Certified Detection}\label{sec:certified detection}

\begin{figure*}[htb]
    \centering
    \includegraphics[width=0.95\linewidth]{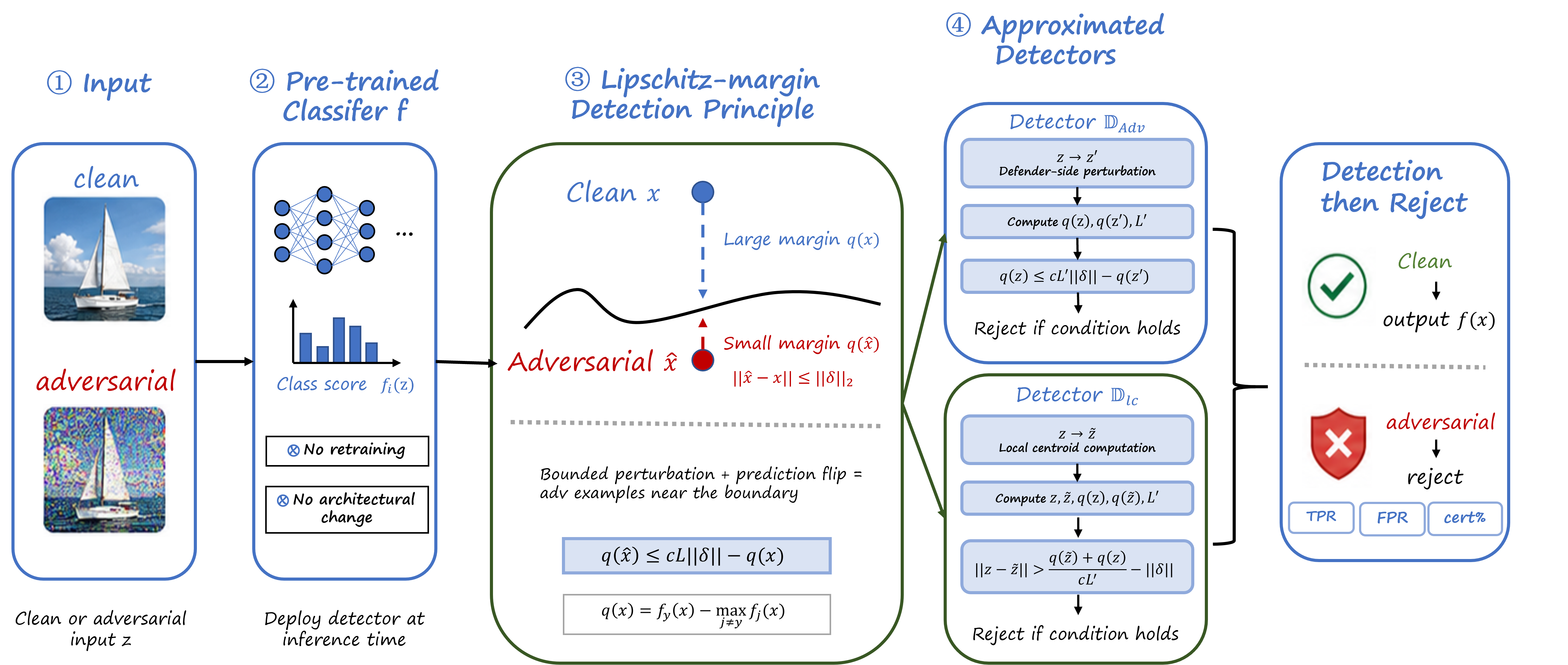}
    \caption{\textbf{Overview of certified adversarial detection.} Taking a pre-trained classifier $f$ as-is (no retraining and no architectural change), we exploit the \emph{Lipschitz--margin} principle: a clean input $\xx$ exhibits a large margin $q(\xx) = f_y(\xx) - \max_{i \neq y} f_i(\xx)$, whereas a successful adversarial example $\adv$ with bounded perturbation $\|\delta\|$ is forced close to the decision boundary, yielding a small margin upper-bounded by $q(\adv) \leq cL\|\delta\| - q(\xx)$ (Lemma~\ref{lemma.certified_detection}). Because the clean counterpart $\xx$ is unobserved at inference time, we instantiate two approximated detectors --- $\detector_{Adv}$ (defender-side perturbation surrogate; Claim~\ref{claim:approximate_adv}) and $\detector_{lc}$ (local-centroid surrogate; Claim~\ref{claim:approximate_lc}) --- which flag adversarial inputs while accepting clean ones, and report \textcolor{blue}{TPR, FPR, and cert\%, where cert\% counts successful in-budget adversarial examples that are both rejected and formally certified.}}
    \label{fig:overview}
\end{figure*}

An $L$-Lipschitz network cannot change its output by more than $L\|\delta\|$ under perturbation $\|\delta\|$. So if a clean input $\xx$ has a large margin in its true-class direction, an in-budget adversarial example $\adv$ cannot simultaneously achieve a large margin in any wrong-class direction --- the budget is too small to ``move'' that much logit mass. Lemma~\ref{lemma.certified_detection} captures this in a one-sided sum bound.

\begin{lemma}(\textbf{Certified Detection})~\label{lemma.certified_detection}
    Given a DNN $f$ with Lipschitz constant $L$ under $\ell_p$ norm $\|\cdot\|_p$, for an adversarial example $\adv$ and its clean counterpart $\xx$, the margin satisfies
    \begin{equation}
        q(\adv)\leq cL\|\delta\|-q(\xx),
    \end{equation}
    where $q(\adv)$ is evaluated w.r.t.\ the predicted class of $\adv$ (Section~\ref{sec:prelim}) and $c=2^{(p-1)/p}$.
\end{lemma}

\noindent The proof is deferred to \prooflocation{proof.adv_margin_bound}. The bound is consistent with certified robustness: $q(\adv)\ge 0$ by definition (Section~\ref{sec:prelim}), so when $\|\delta\|$ is below the certified-robustness threshold $q(\xx)/(cL)$, the right-hand side becomes negative and Lemma~\ref{lemma.certified_detection} cannot be satisfied --- equivalently, no in-budget adversarial example exists, recovering the standard robustness certificate.

Based on Lemma~\ref{lemma.certified_detection}, a certified detector $\detector$ can be formally defined as
\begin{equation}
    \detector (z)=
    \begin{cases}
    1, & \text{if $z$ satisfies Lemma~\ref{lemma.certified_detection}},  \\
    0, & \text{otherwise,}
    \end{cases}
\end{equation}
where input $z$ may be either a clean sample or a perturbed sample. 

However, in practice, the clean counterpart $\xx$ is unknown for an adversarial example, and it does not even exist for a clean sample. Therefore, this detector cannot be directly applied. To resolve this issue, we propose approximation strategies.

\subsection{Approximation to the Certified Detector}\label{sec:approximation to the certified detector}

Approximating the bound in Lemma~\ref{lemma.certified_detection} requires identifying the preimage $\xx$ of $z$, which is generally infeasible because:  
\begin{enumerate}[leftmargin=*, itemsep=0pt, topsep=0pt, parsep=0pt, partopsep=0pt]
    \item For $z = \adv$, the underlying clean input $\xx$ is unknown, as the adversarial perturbation is typically generated through an optimization procedure rather than a well-defined invertible mapping. 
    \item For $z$ that is not adversarial, a corresponding $\xx$ may not exist at all.  
\end{enumerate}

We therefore propose two approximation methods:  
(1) approximating $\xx$ by perturbing $z$;  
(2) introducing a numerically computable surrogate, namely the local centroid.

\subsubsection{Approximation by Perturbation}

We first approximate $\xx$ by perturbing $z$: given $z$, we generate $z' = z + \sigma$ and use it as a surrogate for $\xx$ in Lemma~\ref{lemma.certified_detection} (recall from Section~\ref{sec:prelim} that $z'$ is the defender-side perturbation, not the attacker-side $\adv$).

\begin{claim}(\textbf{Detector $\detector_{Adv}$})\label{claim:approximate_adv}
    For a DNN $f$ and an input $z$ with pseudo-label $\hat y = \arg\max f(z)$, given a defender-side perturbation $z'$ of $z$, if $z=\adv$, then
    \begin{equation}
        q(z;\hat y) + q(z';\hat y) \leq cL'\|\delta\|,
    \end{equation}
    where $L' = (q(z;\hat y) - q(z';\hat y))/(c\,\max(\|z - z'\|,\|\delta\|))$ approximates the Lipschitz constant from the local pair (the numerator is non-negative whenever $z'$ flips $z$'s prediction, so $L'\ge 0$).
\end{claim}

Since $L' \leq L$ and $q(z') \neq q(\xx)$, this introduces a bias compared to Lemma~\ref{lemma.certified_detection}; the inequality in Claim~\ref{claim:approximate_adv} is therefore the \emph{approximate} certification test used by $\detector_{Adv}$ rather than an unconditional bound. Proposition~\ref{prop:approx_certifies} establishes the conditions under which it is provably correct on adversarial inputs.

\subsubsection{Approximation by Local Centroid}

To better track the bias on classifiers with large clean margins, we introduce the concept of a \emph{local centroid}, which can be computed numerically. The local centroid of $\xx$ with respect to $f$ is defined as the local maximizer of the margin function:
\[
\lcx = \arg\max_{x} q(x), \quad q(x) := f_y(x) - \max_{i \neq y} f_i(x).
\]

Such local optima have been used in detecting backdoor attacks~\cite{wang2024mm}. Here, we employ them as surrogates for $\xx$ to strengthen the bound.

\ifcameraready In the gradient-flow limit,\else By Lemma~\ref{lemma:margin_increasing} (Appendix~\ref{appendix:margin_monotonicity}),\fi{} gradient ascent on $q(\cdot;\hat y)$ at the input's frozen pseudo-label $\hat y$ is monotonically non-decreasing; starting from $q(z;\hat y)>0$, the centroid trajectory keeps the margin positive and $\Tilde z$ stays predicted as $\hat y$, so $y(\lcadv)=y(\adv)$ as required by Theorem~\ref{theorem.lowerbound_of_distance}. Following Lyu et al.~\cite{lyu_gradient_2020}, when $f$ is homogeneous in its parameters (a condition satisfied by ReLU networks without bias terms), this ascent additionally converges to a Karush--Kuhn--Tucker (KKT) point of $\min \tfrac{1}{2}\|\xx\|^2$ s.t. $q(\xx)\ge 1$ --- our setting is the input-space counterpart of their parameter-space result.

Intuitively, points near the decision boundary tend to be farther away from their local centroids. This intuition leads to the following guarantee:

\begin{theorem}(\textbf{Certified Detection via Local Centroid})~\label{theorem.lowerbound_of_distance}
Given a neural network $f$ with Lipschitz constant $L$, for an adversarial example $\adv$ and its clean counterpart $\xx$, the distance between $\adv$ and its local centroid $\lcadv$ satisfies
    \begin{equation}
        \|\lcadv-\adv\|\geq \frac{q(\lcadv)+q(\xx)}{cL} - \|\delta\|.
    \end{equation}
\end{theorem}

\noindent The proof\ifcameraready, provided in the full version,\else{} (Appendix~\ref{proof.adv_margin_bound})\fi{} applies Lemma~\ref{lemma.certified_detection} to the pair $(\xx, \lcadv)$ and combines it with the triangle inequality. The theorem strictly enlarges the certified region of Lemma~\ref{lemma.certified_detection}: any pair $(\xx, \adv)$ satisfying the lemma also satisfies the theorem, because $q(\lcadv) \ge q(\adv)$ by construction of the centroid. The right-hand side is non-vacuous (positive) whenever $\|\delta\| < (q(\lcadv) + q(\xx))/(cL)$, which is the regime in which the certified test fires.

Analogous to the first approximation, we further use $q(z')$ as a surrogate for $q(\xx)$ and replace $L$ with a local Lipschitz constant $L'$.

\begin{claim}[\textbf{Detector \( \detector_{lc} \)}]\label{claim:approximate_lc}
    For a DNN $f$ and an input $z$ with pseudo-label $\hat y = \arg\max f(z)$, given a defender-side perturbation $z'$ of $z$ and the local centroid $\lczz$ of $z$, if $z=\adv$, then
    \begin{equation}
        \|\lczz - z\| \geq \frac{q(\lczz;\hat y) + q(z';\hat y)}{cL'} - \|\delta\|.
    \end{equation}
\end{claim}

As with Claim~\ref{claim:approximate_adv}, the inequality is the \emph{approximate} certification test used by $\detector_{lc}$. The substitutions $q(\xx)\to q(z')$ and $L\to L'$ make it conditional on $L'\ge L^*_{(\lcadv,\xx)}$ and on the flip property $q(z';\hat y)\le 0$ (Proposition~\ref{prop:approx_certifies}, part (ii)).

\begin{remark}[False Positive Rate in Practice]
We note that the false positive rate (FPR) of our detectors can be relatively high in practice. This stems from two factors. First, many clean samples naturally lie close to the decision boundary --- an inherent property of the trained model and the dataset --- which longer training or stronger regularization can alleviate (Section~\ref{sec:robust_boosts}). Second, larger perturbation bounds $\|\delta\|$ allow adversarial examples to drift farther from their local centroids, but the detection rule must then admit a wider band that inevitably flags more clean samples. The TPR/FPR trade-off is therefore governed by both the model's clean-margin distribution and the announced budget $\epsilon$.
\end{remark}

\subsection{Approximated Detector}
\label{sec:method}

We now translate Claims~\ref{claim:approximate_adv}--\ref{claim:approximate_lc} into concrete algorithms. The main text describes the detector construction; full pseudocode for perturbed-sample generation, local-centroid optimization, and both detector variants is \ifshowmainalgorithms given below\else deferred to Appendix~\ref{appendix:algorithms}\fi.

\ifshowmainalgorithms
\begin{algorithm}[t]
\caption{Perturbed Sample Generation ($z' = z + \sigma$)}
\label{alg:perturbation}
\begin{algorithmic}
\STATE {\bfseries Input:} Test input $z$ (clean or adversarial), classifier $f$, radius $\|\delta\|$, step size $\alpha$, number of steps $T$
\STATE {\bfseries Output:} Perturbed version $z'$
\STATE Compute pseudo-label $y = \arg\max f(z)$
\STATE Initialize $z' \gets z$
\FOR{$t = 1$ to $T$}
    \STATE Compute loss $\mathcal{L}(f(z'), y)$
    \STATE Compute gradient $g \gets \nabla_{z'} \mathcal{L}(f(z'), y)$
    \STATE Update $z' \gets z' + \alpha \cdot \mathrm{sign}(g)$ \COMMENT{gradient ascent step}
    \STATE Project $z'$ onto the $\ell_2$-ball $\mathcal{B}_2(z,\|\delta\|)$:
    \STATE \hspace{1em} $\Delta \gets z' - z$
    \STATE \hspace{1em} $\Delta \gets \Delta \cdot \min\left\{1,\frac{\|\delta\|}{\|\Delta\|_2}\right\}$
    \STATE \hspace{1em} $z' \gets z + \Delta$
    \STATE Clip pixel values: $z' \gets \mathrm{clip}(z', 0, 1)$
    \IF{$\arg\max f(z') \neq y$}
        \STATE \textbf{break} \COMMENT{break on flip: $y(z')\neq y$, i.e., $q(z';y)\le 0$}
    \ENDIF
\ENDFOR
\end{algorithmic}
\end{algorithm}

\begin{algorithm}[t]
\caption{Approximate Detector from Claim~\ref{claim:approximate_adv} ($\detector_{Adv}$)}
\label{alg:detector-lemma1}
\begin{algorithmic}
\STATE {\bfseries Input:} Input $z$, perturbation radius $\|\delta\|$, model $f$
\STATE {\bfseries Output:} Detection label (\textsc{Adversarial} / \textsc{Clean})
\STATE $\hat y \gets \arg\max f(z)$ \COMMENT{pseudo-label, fixed reference for all $q(\cdot)$ below}
\STATE Generate perturbed version $z' = z + \sigma$
\STATE Estimate $L' \gets \dfrac{q(z;\hat y) - q(z';\hat y)}{c \cdot \max(\|z- z'\|_2, \|\delta\|)}$
\STATE $\tau \gets L' \cdot c \cdot \|\delta\|$
\IF{$q(z;\hat y) + q(z';\hat y) \le \tau$}
    \STATE {\bfseries Return:} \textsc{Adversarial}
\ELSE
    \STATE {\bfseries Return:} \textsc{Clean}
\ENDIF
\end{algorithmic}
\end{algorithm}

\begin{algorithm}[t]
\caption{Local Centroid Optimization}
\label{alg:centroid}
\begin{algorithmic}[1]
\STATE {\bfseries Input:} Input $z$, classifier $f$, step size $\eta$, number of ascent steps $T$
\STATE {\bfseries Output:} Local centroid $\Tilde{z}$ and distance $d_z = \|\Tilde{z} - z\|_2$
\STATE Compute pseudo-label $y \gets \arg\max f(z)$ \COMMENT{frozen throughout ascent}
\STATE Initialize $\Tilde{z} \gets z$
\FOR{$t = 1$ to $T$}
    \STATE Compute margin $q(\Tilde{z}) = f_y(\Tilde{z}) - \max_{j \neq y} f_j(\Tilde{z})$
    \STATE Compute gradient $g \gets \nabla_{\Tilde{z}} q(\Tilde{z})$
    \STATE Update $\Tilde{z} \gets \Tilde{z} + \eta \cdot g$ \COMMENT{Gradient ascent to enlarge margin}
    \STATE Clip $\Tilde{z}$ to valid input range: $\Tilde{z} \gets \mathrm{clip}(\Tilde{z}, 0, 1)$
\ENDFOR
\STATE Compute distance $d_z \gets \|\Tilde{z} - z\|_2$
\end{algorithmic}
\end{algorithm}

\begin{algorithm}[t]
\caption{Approximate Detector from Claim~\ref{claim:approximate_lc} ($\detector_{lc}$)}
\label{alg:detector-adv}
\begin{algorithmic}[1]
\STATE {\bfseries Input:} Input $z$, perturbation budget $\|\delta\|$, model $f$
\STATE {\bfseries Output:} Detection label (\textsc{Adversarial} / \textsc{Clean})
\STATE $\hat y \gets \arg\max f(z)$ \COMMENT{pseudo-label, fixed reference for all $q(\cdot)$ below}
\STATE Generate perturbed version $z' = z + \sigma$
\STATE Compute local centroid $\Tilde{z}$ and distance $d_z = \|\Tilde{z} - z\|_2$
\STATE Estimate $L' \gets \dfrac{q(\Tilde{z};\hat y) - q(z';\hat y)}{c \cdot \max(\|\Tilde{z} - z'\|_2, \|\delta\|)}$
\STATE $\text{RHS} \gets \dfrac{q(\Tilde{z};\hat y) + q(z';\hat y)}{c \cdot L'} - \|\delta\|$
\IF{$d_z \ge \text{RHS}$}
    \STATE {\bfseries Return:} \textsc{Adversarial}
\ELSE
    \STATE {\bfseries Return:} \textsc{Clean}
\ENDIF
\end{algorithmic}
\end{algorithm}

\fi

\subsubsection{Basic Detector ($\detector_{Adv}$)}

\textbf{Perturbed-sample generation.}
Given test input $z$, we generate $z' \in \mathcal{B}_2(z,\|\delta\|)$ by running PGD~\cite{PGD} on the predicted class (Algorithm~\ref{alg:perturbation}). $z'$ replaces the unknown clean counterpart $\xx$ in Claim~\ref{claim:approximate_adv}.

\textbf{Local Lipschitz estimation.}
Since the global Lipschitz constant is typically intractable, we estimate a \emph{local} Lipschitz constant from the pair $(z, z')$ at the input's pseudo-label $\hat y = \arg\max f(z)$:
\[
L' \approx \frac{q(z;\hat y) - q(z';\hat y)}{c \cdot \max(\|z - z'\|_2,\|\delta\|)}, \quad c=\sqrt{2}.
\]
This data-dependent surrogate replaces $L$ in Claim~\ref{claim:approximate_adv}; Algorithm~\ref{alg:detector-lemma1} implements the resulting detector $\detector_{Adv}$.

\subsubsection{Refined Detector ($\detector_{lc}$)}

The Refined Detector is the practical-side counterpart of the Lemma IV.1 $\to$ Theorem IV.2 refinement. In Claim~\ref{claim:approximate_adv}, $z'$ alone plays a double role: it is both the surrogate for $\xx$ in $q(z')$ and the second point used to estimate $L'$. This double-use ties the bias on both quantities to the same geometric distance between $z'$ and the unobserved $\xx$, making the substitution non-negligible whenever $z'$ is not approximately on the segment $(\xx,\adv)$. Introducing the local centroid $\lcx$ (Section~\ref{sec:approximation to the certified detector}) decouples the two roles: $\lcx$ supplies a precisely-computable, high-margin anchor for the centroid-distance bound, while $z'$ continues to drive the local Lipschitz estimator. We compute $\lcx$ by iterative gradient ascent on $q(\cdot)$ starting from $z$ (Algorithm~\ref{alg:centroid}); Algorithm~\ref{alg:detector-adv} then implements $\detector_{lc}$ following Claim~\ref{claim:approximate_lc}.

\subsection{Bias Analysis of Approximated Detector}
\label{sec:bias analysis}

The substitution $q(z') \!\to\! q(\xx)$ contributes a bias bounded by $|q(z')-q(\xx)|$; we focus here on the harder term, the deviation between the local estimator $L'$ and its theoretical counterpart $L$.

\subsubsection{Bias Analysis of the Basic Detector.}
Since Claim~\ref{claim:approximate_adv} replaces $x$ with its perturbed counterpart $z'$, 
the estimated local Lipschitz constant may deviate from the ideal one computed on $(\xx,\adv)$. 
We characterize this deviation as follows.

\begin{lemma}[Bounded Bias of the Approximated Lipschitz Constant]
\label{lemma:bias_analysis}
Let $a=\adv$, $b=\xx$, and $z'=z+\sigma$ denote the adversarial, clean, and perturbed samples, respectively.  
Define $\alpha=\|a-b\|$, $\beta=\|a-z'\|$, and $\gamma=\|b-z'\|$, and let
\[
L^* = \frac{q(a)+q(b)}{c\,\alpha}, \qquad
L' = \frac{|q(a)-q(z')|}{c\,\beta}.
\]
Then the deviation between $L'$ and $L^*$ satisfies
\[
\big|L' - (\alpha/\beta)\,L^*\big|
\;\le\;
\frac{\gamma}{\beta}\,L_{\text{glob}}
+\frac{\alpha}{\beta}\,(L^*-D),
\]
where $D=\frac{|q(a)-q(b)|}{c\,\alpha}\le L^*$ and 
$L_{\text{glob}}$ is a global Lipschitz constant of $q(\cdot)$.
\end{lemma}

The proof is deferred to \prooflocation{appendix:bias_proof}.
Intuitively, the first term $\tfrac{\gamma}{\beta}L_{\text{glob}}$ captures the geometric bias 
introduced by using $z'$ instead of the clean sample $\xx$, 
while the second term $\tfrac{\alpha}{\beta}(L^*-D)$ reflects the margin bias 
between sum- and difference-based margin quotients.
When $z'$ lies approximately along the segment $(\xx,\adv)$, both terms diminish, 
and $L'$ provides a tight local approximation of the ideal Lipschitz constant.

\subsubsection{Bias Analysis of the Refined Detector}

Recall that the ideal constant in Lemma~\ref{lemma.certified_detection} can be written as
\[
L^* = \frac{q(\adv)+q(\xx)}{c\|\adv-\xx\|},
\]
where $\xx$ is the clean input and $\adv$ its adversarial counterpart. In contrast, our detector employs
\[
L' = \frac{|q(\Tilde{z})-q(z')|}{c\|\Tilde{z}-z'\|},
\]
where $z'$ is the perturbed version of $z$ and $\Tilde{z}$ is the local centroid. The difference between these two quantities can be bounded as
\[
L' - L^* \;\le\; (\kappa-1)L_{\text{glob}} + (L_{\text{glob}}-L^*),
\]
\[
\kappa = \frac{\| \Tilde{z} -z\|+\|z'-z\|}{\|\Tilde{z}-z'\|}\,\ge 1,
\]
where $L_{\text{glob}}$ is the global Lipschitz constant. This inequality shows that the bias arises from two sources: (i) the geometric factor $\kappa$, which depends on the relative positions of $\Tilde{z},z,$ and $z'$, and (ii) the inherent looseness of the global bound $L_{\text{glob}}$ compared to $L^*$. In particular, when $z'$ lies close to $\Tilde{z}$, $\kappa$ may become large, leading to a systematic overestimation of $L^*$. Conversely, when $z'$ is well separated, $L'$ better approximates $L^*$.

The above analysis indicates that the bias of the refined detector stems from both geometric and Lipschitz factors. 
To make this relationship more explicit and to enable practical estimation, 
we next derive a deterministic and fully computable upper bound on the deviation between the approximation and the ideal certified term.

\begin{theorem}[Deterministic bound for the approximation]\label{thm:deterministic}
Let $z=\adv$ be an adversarial input with clean counterpart $\xx$ ($\|z-\xx\|\le\|\delta\|$), and $z'=z+\sigma$ with $\|\sigma\|\le\|\delta\|$.
Assume that
\begin{enumerate}[itemsep=0pt, topsep=0pt, parsep=0pt, partopsep=0pt]
\item[\textbf{(A1)}]  \textbf{(Lipschitz of margin)} $q$ is $L_q$-Lipschitz under $\|\cdot\|$ with $L_q\le c\bar L$:
$|q(u)-q(v)|\le L_q\|u-v\|$ for all $u,v$ in a neighborhood of the data manifold.
\item[\textbf{(A2)}] \textbf{ (Global Lipschitz bound )} There exists a computable global Lipschitz \emph{upper bound} $L_{\text{glob}} \le \bar L$, and $L'\le L_{\text{glob}}$ (i.e., the local estimator does not exceed the global upper bound).
% \item[\textbf{(A3)}] \textbf{(Perturbation bound)} $\|\sigma\|\le \|\delta\|$.
\end{enumerate}
Then $0<L'\le L_q/c\le \bar L$ and, for
\[
\Delta_T \;:=\; \Bigg|\;\frac{q(\Tilde{z})+q(z';\hat y)}{c\,L'}\;-\;\frac{q(\Tilde{z})+q(\xx)}{c\,L_{\text{glob}}}\;\Bigg| .
\]
We  write the relative gap $\eta:=1-L'/{\bar L}\in[0,1)$ and have the fully computable bound
\begin{equation}\label{eq:det-eta}
\Delta_T\;\le\;\frac{1}{c}\left[\;\frac{2L_q}{L'}\,\|\delta\| \;+\;\frac{\eta}{\,L'\,}\,\big(q(\Tilde{z})+|q(z';\hat y)|+L_q\,\|\delta\|\big)\;\right].
\end{equation}
\end{theorem}

% \paragraph{Reading the bound.}
$\Delta_T$ is the gap between the approximated certified threshold (computed at inference from $L'$ and $z'$) and the ideal one in Theorem~\ref{theorem.lowerbound_of_distance} (computed from the global $L_{\text{glob}}$ and the unobserved clean margin $q(\xx)$). The two summands isolate two error sources. The first, $\frac{2L_q}{cL'}\|\delta\|$, is the \emph{geometric} stretch from substituting $z'$ for the unobserved $\xx$ (the factor $2$ absorbs the triangle $\xx\to z\to z'$), scaled by how tight $L'$ is relative to the margin's Lipschitz constant $L_q$. The second, scaled by the relative gap $\eta = 1-L'/\bar L \in [0,1)$, is the \emph{looseness} from replacing the global upper bound $\bar L$ with the local estimator $L'$: when $L'$ approaches $\bar L$, $\eta \to 0$ and the second term vanishes. Both terms are observable at inference from $\|\delta\|$, $L'$, $\bar L$, and the margins.

\subsection{When the Practical Detectors Provably Detect Adversarial Examples?}\label{sec:approx_to_cert}

We now connect the approximated detectors $\detector_{Adv}$ and $\detector_{lc}$ back to the formal certificate of Lemma~\ref{lemma.certified_detection} and Theorem~\ref{theorem.lowerbound_of_distance}. Two gaps separate the approximated detectors from the formal certificate: the local estimator $L'$ may fall short of the pairwise Lipschitz $L^*$ (the deviation $|L'-L^*|$ is bounded for $\detector_{Adv}$ by Lemma~\ref{lemma:bias_analysis}, and the resulting threshold gap for $\detector_{lc}$ is bounded by Theorem~\ref{thm:deterministic}), and the surrogate $q(z')$ replaces the unobserved $q(\xx)$ (closed by the break-on-flip rule in Algorithm~\ref{alg:perturbation}, which guarantees $q(z';\hat y)\le 0$ whenever the inner attack succeeds within $T$ steps). To state the certification condition cleanly, we first define the pairwise Lipschitz constants involved.

\begin{prop}[Pairwise Lipschitz]\label{prop:pairwise_suffices}
For an adversarial example $\adv$, its clean counterpart $\xx$, and the local centroid $\lcadv$ of $\adv$, define the \emph{pairwise Lipschitz constants}
\[
L^*_{(\xx,\adv)} \;:=\; \frac{q(\adv) + q(\xx)}{c\,\|\adv - \xx\|},
\qquad
L^*_{(\lcadv,\xx)} \;:=\; \frac{q(\lcadv) + q(\xx)}{c\,\|\lcadv - \xx\|}.
\]
Both are saturations of Lemma~\ref{lemma.certified_detection}: $L^*_{(\xx,\adv)}$ is the minimum $L$ for which the lemma holds on the $(\xx,\adv)$ pair, and $L^*_{(\lcadv,\xx)}$ is the minimum $L$ for which the lemma holds on the $(\lcadv,\xx)$ pair (both $\xx,\lcadv$ have differing predicted classes since $y(\lcadv)=y(\adv)\neq y(\xx)$). Theorem~\ref{theorem.lowerbound_of_distance} can be derived from a \emph{single} application of Lemma~\ref{lemma.certified_detection} on the $(\xx,\lcadv)$ pair plus the triangle inequality (see \prooflocation{proof.adv_margin_bound}); it therefore holds for any $L \geq L^*_{(\lcadv,\xx)}$.
\end{prop}

\begin{prop}[When approximated detectors are certified]\label{prop:approx_certifies}
For every successful in-budget adversarial example $\adv$:
\begin{enumerate}[itemsep=0pt,topsep=2pt,leftmargin=*]
\item[(i)] $L'\ge L^*_{(\xx,\adv)}$ implies the bound of Claim~\ref{claim:approximate_adv} ($\detector_{Adv}$);
\item[(ii)] $L'\ge L^*_{(\lcadv,\xx)}$ together with the flip property $q(z';\hat y)\le 0$ --- guaranteed by Algorithm~\ref{alg:perturbation}'s break-on-flip whenever the inner attack succeeds within $T$ steps --- implies the bound of Claim~\ref{claim:approximate_lc} ($\detector_{lc}$).
\end{enumerate}
\end{prop}

\begin{proof}[Proof sketch]
For input $z=\adv$:
\emph{(i)} The max-clamp denominator in $L'$ combined with $L'\ge L^*_{(\xx,\adv)}$ and Lemma~\ref{lemma.certified_detection} forces $q(z';\hat y)\le -q(\xx)\le 0$, so the flip property emerges as a \emph{consequence}; substituting $q(z';\hat y)\le -q(\xx)$ into the detector's threshold then recovers the Claim~\ref{claim:approximate_adv} bound.
\emph{(ii)} The flip property gives $q(z';\hat y)\le 0\le q(\xx)$, so substituting $q(\xx)\to q(z';\hat y)$ in the detector's right-hand side only loosens it, and $L'\ge L^*_{(\lcadv,\xx)}$ closes the chain back to Theorem~\ref{theorem.lowerbound_of_distance}.
\ifcameraready The full proof is provided in the full version.\else Full proof in Appendix~\ref{appendix:approx_certifies_proof}.\fi
\end{proof}

\begin{remark}[What cert\% certifies]\label{rem:cert_percent}
The framework operates in two regimes.

\textbf{(i) Global Lipschitz $L_{\text{glob}}$.} On networks with a known global Lipschitz constant (e.g., 1-Lipschitz SortNet, Table~\ref{tab:sortnet_certified}), $L_{\text{glob}} \geq L^*_{(\xx,\adv)}$ and $L_{\text{glob}} \geq L^*_{(\lcadv,\xx)}$ for every pair, so plugging $L = L_{\text{glob}}$ certifies both Lemma~\ref{lemma.certified_detection} and Theorem~\ref{theorem.lowerbound_of_distance} for every in-budget adversarial example.

\textbf{(ii) Local estimator $L'$.} On standard pretrained classifiers $L_{\text{glob}}$ is intractable, so we estimate $L'$ at inference. The cert\% metric (Section~\ref{sec:experiments}) measures the empirical fraction of \revisetag{Comment-1}\textcolor{blue}{all successful in-budget adversarial examples that are both rejected and satisfy} the per-detector certification hypothesis of Proposition~\ref{prop:approx_certifies}: $L' \geq L^*_{(\xx,\adv)}$ for $\detector_{Adv}$, and $L' \geq L^*_{(\lcadv,\xx)}$ \emph{together with} the flip property $q(z';\hat y)\le 0$ for $\detector_{lc}$. \textbf{All cert\% numbers reported in the abstract, experiment tables, and discussion use these per-detector definitions}; we keep the shorthand ``$L'\ge L^*$'' in subsequent table captions to avoid clutter. Empirically, on every \textcolor{blue}{rejected} adversarial that satisfies $L'\ge L^*$, $q(z';\hat y)\le 0$ always holds in all our experiments, so the conjoint hypothesis collapses to $L'\ge L^*$ in practice. By Proposition~\ref{prop:approx_certifies}, \textcolor{blue}{cert\% is exactly the fraction on which the approximated detectors provide a formal certificate.}
\end{remark}

\section{Experiments}
\label{sec:experiments}
\subsection{Experimental Setup}
We evaluate on \textbf{MNIST}~\cite{lecun1998mnist} (LeNet-5), \textbf{CIFAR-10}~\cite{CIFAR10} (VGG16), and \textbf{ImageNet-1k}~\cite{ImageNet} (PyTorch-pretrained ResNet-50; ViT-B/16 in the A\textsuperscript{2}D comparison). Adversarial examples are generated under $\ell_2$ with FGSM~\cite{FGSM}, BIM~\cite{IFGSM}, PGD ($40$ steps, $10$ restarts)~\cite{PGD}, and CW ($40$ steps)~\cite{carlini2017towards}.

\textbf{Metrics.}
We report \textbf{TPR}, the fraction of \emph{successful} adversarial examples that are detected (success means the prediction changes), \textbf{FPR}, the fraction of clean samples flagged as adversarial, and \textbf{cert\%}. cert\% is reported only on detectors whose guarantee depends on $L'$ (i.e., $\detector_{Adv}$ and $\detector_{lc}$).

\textbf{Implementation.}
Both detectors are implemented in PyTorch on top of the pretrained classifiers without any retraining or fine-tuning. The perturbed surrogate $z'$ is generated by $T = 10$ BIM steps with step size $\alpha = \epsilon/T$ inside the $\ell_2$-ball of radius $\|\delta\|$ (Algorithm~\ref{alg:perturbation}); the local centroid $\Tilde{z}$ uses $T_{\text{cent}} = 10$ gradient-ascent steps on $q(\cdot)$ at step size $0.1$ with $[0,1]^d$ clipping (Algorithm~\ref{alg:centroid}). Unless otherwise noted, all results are evaluated on a single NVIDIA V100 GPU, averaged over 5 runs with different random seeds.

\textbf{Detection Baselines.}\label{sec:4.1.2:randomized smoothing and baselines}
We compare against five external baselines spanning certified and empirical detection.
\textcolor{blue}{Certified detection remains a relatively nascent setting, so the baselines differ in what they certify.\revisetag{Review-C-3} CTT~\cite{shumailov2020towards} is the closest prior certified-detection method, but it provides a complementary specificity-side certificate rather than our recall-side guarantee in Definition~\ref{def:certified_detector}. For a broader defense-level comparison, we therefore instantiate certified-robustness methods (RS, SortNet, CRM) as detectors via their certified radii, and include state-of-the-art empirical detectors (ContraNet and A\textsuperscript{2}D).}

\textit{(i) SortNet}~\cite{zhang_rethinking_2022} is a 1-Lipschitz network whose certified robustness radius gives a natural detection rule (flag inputs whose certified $\ell_2$ radius $r < \epsilon$). SortNet doubles as the testbed for our fully certified Lemma~\ref{lemma.certified_detection} (Table~\ref{tab:sortnet_certified}), since $L_{\text{glob}}$ is exactly known and used directly.

\textit{(ii) CRM}~\cite{fazlyab2023certified} is a Lipschitz-regularised classifier with dynamic margin maximisation; its certified-radius detection rule mirrors SortNet's.

\textit{(iii) CTT}~\cite{shumailov2020towards} provides the complementary specificity-side certified-detection baseline; we evaluate both CTT-loose and CTT-strict thresholding variants.

\textit{(iv) ContraNet}~\cite{yang2022you} trains a class-conditional generator and detects adversarial inputs via reconstruction error; it does not scale to ImageNet-1k.

\textit{(v) A\textsuperscript{2}D}~\cite{zhao2024attack} is a recent state-of-the-art empirical detector that uses the perturbation budget required to flip the prediction as the detection signal; it is one of the few empirical detectors evaluated on full ImageNet-1k and serves as our headline empirical comparison (Table~\ref{tab:vit_a2d}).

In addition to these five external baselines, we adapt the randomized smoothing of Cohen \emph{et al.}~\cite{cohen2019certified} into a certified detection rule for completeness: compute the certified radius $r$ for each input and output \textsc{Adversarial} iff $r < \epsilon$. We use $n_0 = 64$, $n = 512$, and $\sigma = 0.25$ throughout the main body; the conservative variant (smoothed-mismatch + certified radius) is reported in the Conservative RS detector paragraph.

% Order matters: \input order determines table numbering. The SortNet
% headline table is referenced first in §5.2, so it must come first.

\subsection{Main Results}
We evaluate our two approximated detectors $\detector_{Adv}$ and $\detector_{lc}$ on MNIST, CIFAR-10, and ImageNet-1k.
Across datasets, our methods match or outperform existing detection baselines on the TPR/FPR frontier.
Unlike RS, CTT, or ContraNet, our detectors require neither Monte Carlo sampling, retraining, nor additional models to distinguish adversarial examples.

% Source for D_Adv / D_lc: results_ndss/sortnet/{advref,lemma}/sortnet/eps_*/{bim,test}_stats.json.
% TPR on BIM advs (200 successful per eps); FPR on 1000 clean inputs;
% cert% from detect_stats.cert_rate. The fully-certified path (Lem 1 with
% L_{glob}) trivially has cert% = 100% so cert% is reported only for
% the approximated detectors that depend on L'.
\begin{table*}[t]
\centering
\caption{\textbf{\textcolor{blue}{Fully certified detection} on 1-Lipschitz SortNet (CIFAR-10).}
With $L_{\text{glob}}$ exactly known, Lemma~\ref{lemma.certified_detection} (using $L=L_{\text{glob}}$) achieves $100\%$ \emph{certified} detection at FPR substantially below the baseline detectors built on SortNet's and CRM's certified radii. Our approximated detectors $\detector_{Adv}$ and $\detector_{lc}$ use the local estimator $L'$ and reduce FPR further; \textcolor{blue}{we report cert\% as the fraction of all successful in-budget adversarial examples that are both rejected and formally certified.}}
\label{tab:sortnet_certified}
\setlength{\tabcolsep}{3pt}
\begin{tabular*}{\linewidth}{@{\extracolsep{\fill}} ccccccccccccc}
\toprule
$\epsilon$
& \multicolumn{2}{c}{SortNet~\cite{zhang_rethinking_2022}}
& \multicolumn{2}{c}{CRM~\cite{fazlyab2023certified}}
& \multicolumn{2}{c}{Lem~\ref{lemma.certified_detection}, $L_{\text{glob}}$ (ours)}
& \multicolumn{3}{c}{$\detector_{Adv}$ (ours)}
& \multicolumn{3}{c}{$\detector_{lc}$ (ours)} \\
\cmidrule(lr){2-3} \cmidrule(lr){4-5} \cmidrule(lr){6-7} \cmidrule(lr){8-10} \cmidrule(lr){11-13}
& TPR & FPR & TPR & FPR & TPR & FPR & TPR & FPR & cert\% & TPR & FPR & cert\% \\
\midrule
0.1 & 100.0 & 22.0  & 100.0 & 25.2  & 100.0 & 10.5 &  99.0 &  9.3 & 58.1 & 100.0 & 10.1 & 36.5 \\
0.3 & 100.0 & 59.0                & 100.0 & 68.0  & 100.0 & 31.5               & 100.0 & 21.3 & 59.0 & 100.0 & 25.8 & 41.0 \\
0.5 & 100.0 & 86.0                & 100.0 & 96.0  & 100.0 & 58.5               &  99.5 & 35.3 & 65.3 & 100.0 & 40.7 & 50.0 \\
1.0 & 100.0 & 100.0               & 100.0 & 100.0 & 100.0 & 95.0 & 100.0 & 59.0 & 71.5 & 100.0 & 63.5 & 57.5 \\
\bottomrule
\end{tabular*}
\end{table*}

\paragraph{Fully Certified Detection on Lipschitz Networks.}
On 1-Lipschitz SortNet~\cite{zhang_rethinking_2022}, where $L_{\text{glob}}$ is exactly known, applying Lemma~\ref{lemma.certified_detection} with $L=L_{\text{glob}}$ yields \emph{fully} certified detection: every successful $\ell_2$ adversarial example within budget $\epsilon$ is provably rejected (TPR $=100\%$).
Table~\ref{tab:sortnet_certified} compares this fully certified path against the baseline detectors built on top of SortNet's and CRM's~\cite{fazlyab2023certified} own certified-radius computations (flag if the certified radius $r < \epsilon$), and against our approximated detectors $\detector_{\text{Adv}}$ and $\detector_{\text{lc}}$ that use the local Lipschitz estimator $L'$ in place of $L_{\text{glob}}$.
All methods reach $100\%$ TPR (Except $\detector_{\text{Adv}}$ with TPR $\ge 99\%$); the differentiator is FPR.
Our fully certified path delivers substantially lower FPR than either baseline (e.g., $31.5\%$ vs.\ $59.0\%$ / $68.0\%$ at $\epsilon{=}0.3$). \revisetag{Comment-1}\textcolor{blue}{The practical $\detector_{\text{Adv}}$ and $\detector_{\text{lc}}$ use the same detection-rule families with local $L'$ estimates, which can further tighten FPR but are certified only on samples where the per-detector $L'$ condition holds.}

% Source: results_ndss/mnist/{advref,lemma}/eps_*/{bim,test}_stats.json.
\begin{table}[t]
\centering
\small
\caption{\textbf{MNIST (LeNet-5): detection at representative $\ell_2$ budgets.}
TPR on BIM; FPR on clean inputs; \textcolor{blue}{cert\% is the fraction of all successful in-budget adversarial examples that are both rejected and formally certified ($L' \geq L^*$).} RS (Cohen) reported at $\sigma{=}0.25$. Full per-attack and per-$\sigma$ results in Appendix~\ref{appendix:full_experiments}, Tables~\ref{tab:mnist-tpr}--\ref{tab:mnist-clean-fpr}. }
\label{tab:mnist_compact}
\setlength{\tabcolsep}{3pt}
\begin{tabular}{ccccccccc}
\toprule
$\epsilon$
& \multicolumn{3}{c}{$\detector_{Adv}$ (ours)}
& \multicolumn{3}{c}{$\detector_{lc}$ (ours)}
& \multicolumn{2}{c}{RS (Cohen)} \\
\cmidrule(lr){2-4} \cmidrule(lr){5-7} \cmidrule(lr){8-9}
& TPR & FPR & cert\%
& TPR & FPR & cert\%
& TPR & FPR \\
\midrule
0.1 & 100.0 & 0.5  & 92.0 & 100.0 & 1.0  & 90.0 & 44.5  & 4.0  \\
0.5 & 100.0 & 8.0  & 94.0               & 100.0 & 8.5  & 86.0               & 100.0 & 22.0 \\
1.0 & 100.0 & 17.0 & 97.0               & 100.0 & 17.5 & 80.0               & 100.0 & 98.0 \\
2.0 & 100.0 & 47.0 & 98.0 & 100.0 & 47.5 & 75.0 & 100.0 & 98.0 \\
\bottomrule
\end{tabular}
\end{table}

\textbf{MNIST.}
Both $\detector_{Adv}$ and $\detector_{lc}$ achieve perfect detection on MNIST, maintaining $100\%$ TPR across all attacks (FGSM, BIM, PGD, CW) at every $\epsilon$ tested (Table~\ref{tab:mnist_compact}). FPR stays at most $1\%$ for $\epsilon \le 0.1$ and grows gradually thereafter ($17\%$ at $\epsilon{=}1.0$). The Cohen detector by contrast attains $\le 45\%$ TPR at $\epsilon{=}0.1$ and reaches $98\%$ FPR by $\epsilon{=}1.0$ at the same noise level. Full per-attack and per-$\sigma$ breakdowns are in Appendix~\ref{appendix:full_experiments}.

% Rerun source (NDSS 2026): results_ndss/cifar10/standard/{advref,lemma}/vgg16/eps_*/bim_stats.json
% TPR on BIM advs (~500 successful adv per eps); FPR on full 10000-image CIFAR-10 test set;
% cert% from detect_stats.cert_rate using L*(xc, x_clean) for D_lc and L*(x, xadv) for D_Adv.
\begin{table*}[t]
\centering
\caption{\textbf{CIFAR-10: detection at representative $\ell_2$ budgets on VGG16.}
TPR on BIM attacks (averaged over $\sim$500 successful adversarials per $\epsilon$); FPR on the full CIFAR-10 test set; \textcolor{blue}{cert\% is the fraction of all successful in-budget adversarial examples that are both rejected and formally certified ($L' \geq L^*$).} RS (Conservative) is the strengthened RS variant that combines smoothed-classifier mismatch with the certified radius (Section~\ref{sec:experiments}). Full per-attack results in Appendix~\ref{appendix:full_experiments}.}
\label{tab:cifar_compact}
\setlength{\tabcolsep}{4pt}
\begin{tabular*}{\linewidth}{@{\extracolsep{\fill}} ccccccccccccc}
\toprule
$\epsilon$
& \multicolumn{3}{c}{$\detector_{Adv}$ (ours)}
& \multicolumn{3}{c}{$\detector_{lc}$ (ours)}
& \multicolumn{2}{c}{RS (Cohen)}
& \multicolumn{2}{c}{RS (Conservative)}
& \multicolumn{2}{c}{ContraNet} \\
\cmidrule(lr){2-4} \cmidrule(lr){5-7} \cmidrule(lr){8-9} \cmidrule(lr){10-11} \cmidrule(lr){12-13}
& TPR & FPR & cert\%
& TPR & FPR & cert\%
& TPR & FPR
& TPR & FPR
& TPR & FPR \\
\midrule
0.05 & 100.0 & 15.3 & 63.2 & 100.0 & 15.6 & 62.4 & 38.5  & 46.0  & 93.5  & 93.2  & 87.0 & 9.0 \\
0.1  & 100.0 & 29.5 & 76.4 & 100.0 & 29.9 & 75.0 & 53.0  & 60.2  & 95.0  & 93.8  & 87.8 & 9.0 \\
0.3  & 100.0 & 74.8 & 83.4 & 100.0 & 75.8 & 80.4 & 95.8  & 96.0  & 99.6  & 99.2  & 74.6 & 9.0 \\
0.5  & 100.0 & 92.7 & 82.6 & 100.0 & 93.6 & 78.8 & 100.0 & 100.0 & 100.0 & 100.0 & 85.0 & 9.0 \\
\bottomrule
\end{tabular*}
\end{table*}

\textbf{CIFAR-10.}
Table~\ref{tab:cifar_compact} reports representative results on CIFAR-10 (VGG16). Both $\detector_{Adv}$ and $\detector_{lc}$ reach $100\%$ TPR at every $\epsilon$, with FPR rising from $\sim 15\%$ at $\epsilon{=}0.05$ to $\sim 93\%$ at $\epsilon{=}0.5$. \revisetag{Review-A-1\tagsep\mbox{Review-C-4}}\textcolor{blue}{On a standard-trained VGG16, $\epsilon \leq 0.1$ gives a practical operating point with FPR $\leq 30\%$. We return to the FPR--budget trade-off and ways to shift this frontier in
  Section~\ref{sec:fpr_eps}.}
In contrast, RS (Cohen) maintains moderate TPR but incurs very high FPR ($35$--$100\%$) even at small perturbations. \textcolor{blue}{Our cert\% at $\epsilon{=}0.1$ on VGG16 reaches $75\%$, while empirical TPR remains $100\%$. \revisetag{\mbox{Comment-1}\tagsep\mbox{Review-B-5}}Compared with ContraNet, our detectors achieve $100\%$ empirical TPR and additionally certify the cert\% subset, whereas ContraNet provides only empirical separation. Additional CIFAR-10 ResNet-18 and MNIST VGG-style CNN results in Appendix~\ref{appendix:full_experiments}, Table~\ref{tab:multi_classifier}, further validate that the trend is not tied to the primary classifier.} Full per-attack tables are in Appendix~\ref{appendix:full_experiments}.

% Source: results_ndss/imagenet/{advref,lemma}/eps_*/{bim,test}_stats.json (ResNet-50)
% and results_ndss/imagenet_vit/{advref,lemma}/eps_*/ (ViT-B/16).
\begin{table*}[t]
\centering
\caption{\textbf{ Detection on ImageNet-1k (500 images) with PyTorch pretrained models.}
TPR on $\ell_2$ BIM attacks; FPR on clean samples; \textcolor{blue}{cert\% is the fraction of all successful in-budget adversarial examples that are both rejected and formally certified ($L' \geq L^*$).} The full $\epsilon \in [0.1, 1.0]$ breakdown for ResNet-50 is in Appendix~\ref{appendix:full_experiments}, Table~\ref{tab:imagenet_cert_det_full}. }
\label{tab:imagenet_cert_det}
\begin{subtable}[t]{0.49\textwidth}
\centering
\caption{ResNet-50 pretrained model.}
\label{tab:imagenet_resnet50}
\setlength{\tabcolsep}{3pt}
\begin{tabular}{ccccccccc}
\toprule
$\epsilon$
& \multicolumn{3}{c}{$\detector_{Adv}$ (ours)}
& \multicolumn{3}{c}{$\detector_{lc}$ (ours)}
& \multicolumn{2}{c}{RS (Cohen)} \\
\cmidrule(lr){2-4} \cmidrule(lr){5-7} \cmidrule(lr){8-9}
& TPR & FPR & cert\%
& TPR & FPR & cert\%
& TPR & FPR \\
\midrule
0.1 & 99.5 &  9.0 & 20.4 & 99.7 &  9.4 & 63.8 &  34.0 &  17.0 \\
0.3 & 99.8 & 24.6 & 42.9 & 99.6 & 23.7 & 72.7 &  62.5 &  36.0 \\
0.5 & 99.6 & 34.4 & 38.1 & 99.3 & 32.8 & 80.4 &  85.5 &  51.0 \\
1.0 & 99.8 & 54.8 & 41.9 & 99.8 & 53.5 & 76.3 & 100.0 & 100.0 \\
\bottomrule
\end{tabular}
\end{subtable}
\hfill
\begin{subtable}[t]{0.49\textwidth}
\centering
\caption{ViT-B/16 pretrained model (vs.\ A\textsuperscript{2}D~\cite{zhao2024attack}).}
\label{tab:vit_a2d}
\setlength{\tabcolsep}{3pt}
\begin{tabular}{ccccccccc}
\toprule
$\epsilon$
& \multicolumn{3}{c}{$\detector_{Adv}$ (ours)}
& \multicolumn{3}{c}{$\detector_{lc}$ (ours)}
& \multicolumn{2}{c}{A\textsuperscript{2}D} \\
\cmidrule(lr){2-4} \cmidrule(lr){5-7} \cmidrule(lr){8-9}
& TPR & FPR & cert\%
& TPR & FPR & cert\%
& TPR & FPR \\
\midrule
0.1 & 97.4               & 2.8                & 35.1               & 100.0 & 3.6  & 63.0               & 89.5 & 11.2 \\
0.3 & 97.5               & 6.5                & 38.0               &  99.0 & 8.2  & 67.3               & 90.3 & 11.2 \\
0.5 & 97.0               & 14.0               & 39.8               &  99.0 & 15.4 & 72.3               & 90.3 & 11.2 \\
1.0 & 98.0               & 25.0               & 58.8               &  99.3 & 26.6 & 77.5               & 93.1 & 11.2 \\
\bottomrule
\end{tabular}
\end{subtable}
\end{table*}

\textbf{ImageNet-1k.}
Table~\ref{tab:imagenet_cert_det} reports representative $\epsilon$ values on PyTorch pretrained ResNet-50; the full $\epsilon \in [0.1, 1.0]$ at step $0.1$ is in Appendix~\ref{appendix:full_experiments}, Table~\ref{tab:imagenet_cert_det_full}. Both $\detector_{Adv}$ and $\detector_{lc}$ achieve $\geq 98.8\%$ TPR while keeping FPR within $9$--$55\%$ even at $\epsilon{=}1.0$, where the Cohen detector saturates to $100\%$ FPR. \textcolor{blue}{On the certificate, $\detector_{lc}$ formally certifies $64$--$80\%$ of all successful in-budget adversarial examples while empirically rejecting almost all of them. \revisetag{Comment-1}The magnitude is comparable to reported RS certified-accuracy rates on the same backbone, but the metrics measure different events and are not directly interchangeable.}

\paragraph{ViT-B/16 vs.\ a SOTA empirical detector (A\textsuperscript{2}D).}
We further evaluate on a PyTorch pretrained ViT-B/16 and compare against A\textsuperscript{2}D~\cite{zhao2024attack}, a recent state-of-the-art empirical adversarial detector. Table~\ref{tab:vit_a2d} (right half of Table~\ref{tab:imagenet_cert_det}) shows that our $\detector_{lc}$ reaches $100\%$ TPR at $3.6\%$ FPR for $\epsilon{=}0.1$ on ViT-B/16, while A\textsuperscript{2}D plateaus at $89.5\%$ TPR with $11.2\%$ FPR. $\detector_{lc}$ maintains a higher TPR than A\textsuperscript{2}D at every $\epsilon$ and a lower FPR at the small budgets ($\epsilon\le 0.3$) most relevant to security deployments; unlike A\textsuperscript{2}D, which provides no formal guarantee, \revisetag{Comment-1}\textcolor{blue}{$\detector_{lc}$ additionally certifies at least $63\%$ of all successful in-budget adversarial examples.}

\paragraph{Conservative RS detector.}
The ``RS (Conservative)'' columns of Table~\ref{tab:cifar_compact} report a stronger RS variant that combines two certified-robustness checks: (i) compare the base classifier's prediction $f(z)$ to the smoothed prediction $f_s(z)$ and flag mismatches; (ii) for those flagged inputs compute the certified radius $r$~\cite{cohen2019certified} and output \textsc{Adversarial} iff $r < \epsilon$. The conservative variant raises TPR to $93$--$100\%$, but FPR exceeds $93\%$ across every $\epsilon$, effectively rejecting nearly every clean input. Even a strengthened RS detector therefore retains a substantially worse safety--usability trade-off than our certified detection.

\paragraph{Comparison with CTT.}
CTT~\cite{shumailov2020towards} is the closest prior certified-detection framework, but it certifies a complementary specificity-side property. Under $\ell_2$ FGSM and PGD on MNIST and CIFAR-10, our detectors match or beat CTT in TPR at comparable FPR without retraining; the full comparison is in Appendix~\ref{appendix:full_experiments}, Table~\ref{tab:ctt_comparison}.

\begin{table}[htb]
\centering
\caption{Average per-sample detection time (seconds, mean $\pm$ std over $n{=}500$ inputs). Experiments are conducted on a single NVIDIA V100 GPU.}
\label{tab:runtime}
\scriptsize
\setlength{\tabcolsep}{4pt}
\begin{tabular*}{\linewidth}{@{\extracolsep{\fill}}lccc}
\toprule
Method & MNIST & CIFAR-10 & ImageNet-1k \\
\midrule
$\detector_{Adv}$ (base) & 0.020 $\pm$ 0.002 & 0.071 $\pm$ 0.023 & 0.190 $\pm$ 0.073 \\
\textcolor{blue}{$\detector_{Adv}$ (strong $L'$)} & \textcolor{blue}{0.048 $\pm$ 0.002} & \textcolor{blue}{0.183 $\pm$ 0.017} & \textcolor{blue}{0.451 $\pm$ 0.084} \\
$\detector_{lc}$ (base)  & 0.063 $\pm$ 0.006 & 0.151 $\pm$ 0.026 & 0.421 $\pm$ 0.074 \\
\textcolor{blue}{$\detector_{lc}$ (strong $L'$)} & \textcolor{blue}{0.155 $\pm$ 0.008} & \textcolor{blue}{0.341 $\pm$ 0.036} & \textcolor{blue}{0.765 $\pm$ 0.087} \\
RS (Cohen)               & 0.335 $\pm$ 0.003 & 0.437 $\pm$ 0.132 & 3.757 $\pm$ 0.502 \\
\bottomrule
\end{tabular*}
\end{table}

\textbf{Trends across perturbation scales.}
\revisetag{Comment-1}\textcolor{blue}{Empirical} TPR climbs to $100\%$ within a small budget on all datasets; FPR rises smoothly with $\epsilon$, as the detection band must widen to admit the larger adversarial radius. Full perturbation-scale curves are in Appendix~\ref{appendix:full_experiments}, Figure~\ref{fig:summary_curves}. The per-sample detection-statistic distribution for $\detector_{lc}$ is also reported in Appendix~\ref{appendix:full_experiments}, Figure~\ref{fig:rhs_dist}; the clean and adversarial distributions are clearly separated across FGSM, PGD, BIM, and CW. \textcolor{blue}{Appendix~\ref{appendix:full_experiments}, Table~\ref{tab:lprime_ratio} further reports the realised $L'/L^*$ distribution on CIFAR-10. \revisetag{\mbox{Review-A-3}\tagsep\mbox{Review-C-2}}Most samples concentrate near or above the sufficient threshold $1$, and the below-threshold cases remain close to $1$, explaining why the practical detector keeps high empirical TPR while cert\% captures only the formally certified subset.}

\subsection{\textcolor{blue}{Stronger $L'$ Estimation via Multi-Probe Search}}\label{sec:strong_estimator}

\textcolor{blue}{To increase certified coverage without changing the detectors, we replace the single-probe local estimate with $L'_{\mathrm{strong}}=\max_k L'_k$.\revisetag{\mbox{Comment-1}\tagsep\mbox{Review-A-3}\tagsep\mbox{Review-A-6}\tagsep\mbox{Review-A-7}\tagsep\mbox{Review-C-2}} A larger $L'$ makes $L'\ge L^*$ easier to satisfy but also enlarges the rejection threshold. Table~\ref{tab:strong_estimator} shows this trade-off: cert\% reaches $99.5$--$100\%$ on MNIST, gains $14.5$--$34.8$ points on CIFAR-10, and gains up to $66.3$ points on ImageNet, with at most $7.3$ FPR points added and negligible TPR change. Table~\ref{tab:runtime} gives cross-dataset latency; Appendix~\ref{appendix:full_experiments}, Table~\ref{tab:strong_estimator_cost} shows diminishing returns beyond $K{=}1,T{=}6$.}

\begin{table*}[t]
\centering
\scriptsize
\caption{\textbf{\textcolor{blue}{Stronger $L'$ estimation improves certified coverage.}}
\textcolor{blue}{Only the local estimate is replaced by a multi-probe $L'_{\mathrm{strong}}$ ($K{=}1,T{=}6$ unless stated), trading cert\% gains for FPR/runtime.}}
\label{tab:strong_estimator}
\setlength{\tabcolsep}{3.5pt}
\begin{tabular*}{\linewidth}{@{\extracolsep{\fill}} llcccccc}
\toprule
\multirow{2}{*}{Dataset} & \multirow{2}{*}{$\epsilon$} &
\multicolumn{3}{c}{$\detector_{Adv}$} &
\multicolumn{3}{c}{$\detector_{lc}$} \\
\cmidrule(lr){3-5} \cmidrule(lr){6-8}
 & & $\Delta$TPR & $\Delta$FPR & cert\% base$\rightarrow$strong & $\Delta$TPR & $\Delta$FPR & cert\% base$\rightarrow$strong \\
\midrule
\multirow{2}{*}{MNIST}
 & 0.50 & +0.0 & +0.9 & 92.0$\rightarrow$100.0 & +0.0 & +1.5 & 94.0$\rightarrow$100.0 \\
 & 1.00 & +0.0 & +2.3 & 97.0$\rightarrow$99.5 & +0.0 & +2.5 & 98.0$\rightarrow$100.0 \\
\midrule
\multirow{4}{*}{CIFAR-10}
 & 0.05 & +0.0 & +0.4 & 63.2$\rightarrow$98.0 & +0.0 & +1.0 & 62.4$\rightarrow$87.0 \\
 & 0.10 & +0.0 & +2.0 & 76.4$\rightarrow$93.5 & +0.0 & +2.0 & 75.0$\rightarrow$89.5 \\
 & 0.30 & +0.0 & +5.8 & 83.4$\rightarrow$98.0 & +0.0 & +7.3 & 80.4$\rightarrow$97.0 \\
 & 0.50 & +0.0 & +5.7 & 82.6$\rightarrow$100.0 & +0.0 & +5.0 & 78.8$\rightarrow$100.0 \\
\midrule
\multirow{2}{*}{ImageNet}
 & 0.10 & +0.0 & +1.0 & 20.2$\rightarrow$86.5 & +0.3 & +1.9 & 64.0$\rightarrow$82.0 \\
 & 0.50 & +0.2 & +4.7 & 38.2$\rightarrow$95.0 & +0.5 & +5.8 & 80.8$\rightarrow$89.6 \\
\bottomrule
\end{tabular*}
\end{table*}

\subsection{\textcolor{blue}{Adaptive Attack Evaluation}}\label{sec:adaptive_attack}

\textcolor{blue}{\paragraph{Adaptive strategy.}
The detectors fire when the observed margin is \emph{too small} for the declared budget (Lemma~\ref{lemma.certified_detection}). \revisetag{Comment-2} A detector-aware attacker therefore seeks a misclassification with the largest possible wrong-class margin. Since $q(\xx)$ is fixed, maximizing the left-hand side of Lemma~\ref{lemma.certified_detection}, $q(\xx+\delta)+q(\xx)$, is exactly maximizing $q(\xx+\delta)$:
\[
  \max_{\|\delta\|_2\le\epsilon}\; q(\xx+\delta)+q(\xx),
\]
using $100$ Adam steps with $\ell_2$ projection and $[0,1]^d$ clipping. We write $\advadp=\xx+\delta$ for the attacker-side input, distinct from the defender-side surrogate $z'$.}

\textcolor{blue}{We report the classification attack success rate before detection as ASR$_{\mathrm{cls}}$, and the end-to-end evasion rate after detection as ASR$_{\mathrm{joint}}=\mathrm{ASR}_{\mathrm{cls}}\cdot(1-\mathrm{TPR})$.}

\begin{table*}[t]
\centering
\scriptsize
\caption{\textbf{\textcolor{blue}{Detector-aware adaptive attack evaluation.}}
\textcolor{blue}{ASR$_{\mathrm{joint}}=\mathrm{ASR}_{\mathrm{cls}}\cdot(1-\mathrm{TPR})$ is end-to-end evasion after detection; cert\% is shown before/after stronger $L'$.}}
\label{tab:adaptive_attack}
\setlength{\tabcolsep}{4pt}
\begin{tabular*}{\linewidth}{@{\extracolsep{\fill}} llcccccccc}
\toprule
\multirow{2}{*}{Dataset} & \multirow{2}{*}{$\epsilon$} &
\multirow{2}{*}{BIM ASR} & \multirow{2}{*}{Adp. ASR} &
\multicolumn{3}{c}{$\detector_{Adv}$} &
\multicolumn{3}{c}{$\detector_{lc}$} \\
\cmidrule(lr){5-7} \cmidrule(lr){8-10}
 & & & & TPR & ASR$_{\mathrm{joint}}$ & cert\% base$\rightarrow$strong & TPR & ASR$_{\mathrm{joint}}$ & cert\% base$\rightarrow$strong \\
\midrule
\multirow{4}{*}{CIFAR-10}
 & 0.05 & 8.9 & 11.4 & 100.0 & 0.0 & 3.0$\rightarrow$96.5 & 100.0 & 0.0 & 5.0$\rightarrow$67.5 \\
 & 0.10 & 20.3 & 26.1 & 100.0 & 0.0 & 9.0$\rightarrow$98.0 & 100.0 & 0.0 & 8.5$\rightarrow$77.5 \\
 & 0.30 & 61.0 & 80.7 & 100.0 & 0.0 & 16.5$\rightarrow$97.5 & 100.0 & 0.0 & 17.0$\rightarrow$97.0 \\
 & 0.50 & 78.5 & 98.5 & 100.0 & 0.0 & 24.0$\rightarrow$99.5 & 100.0 & 0.0 & 30.5$\rightarrow$99.0 \\
\midrule
\multirow{4}{*}{ImageNet}
 & 0.10 & 4.7 & 14.1 & 87.4 & 1.8 & 2.8$\rightarrow$89.1 & 99.0 & 0.1 & 11.6$\rightarrow$92.6 \\
 & 0.30 & 14.3 & 33.6 & 91.6 & 2.8 & 7.4$\rightarrow$92.8 & 97.2 & 0.9 & 20.0$\rightarrow$95.6 \\
 & 0.50 & 23.8 & 46.6 & 94.0 & 2.8 & 12.6$\rightarrow$92.5 & 98.0 & 0.9 & 25.4$\rightarrow$95.7 \\
 & 1.00 & 45.8 & 70.3 & 98.8 & 0.8 & 20.5$\rightarrow$94.1 & 99.6 & 0.3 & 33.2$\rightarrow$96.4 \\
\bottomrule
\end{tabular*}
\end{table*}

\textcolor{blue}{Table~\ref{tab:adaptive_attack} shows that the adaptive objective raises classification ASR over BIM, especially on ImageNet, yet end-to-end evasion remains low: ASR$_{\mathrm{joint}}=0$ on CIFAR-10 and at most $2.8\%$/$0.9\%$ for $\detector_{Adv}$/$\detector_{lc}$ on ImageNet. The stronger multi-probe estimate restores much of the certified subset.}

\subsection{Robustness-Oriented Training Improves Detection}\label{sec:robust_boosts}
\begin{table*}[t]
\centering
\caption{\textbf{Robust training boosts our detection on CIFAR-10 (VGG16).}
TPR (\%) on $\ell_2$ PGD attacks; FPR (\%) on clean inputs. Two robust-training primitives applied to the same VGG16 base model: $\ell_2$ PGD adversarial training (radius $0.5$, AT columns) and Gaussian augmentation ($\sigma{=}0.25$). RS baselines do not apply to PGD-adversarial-trained models.}
\label{tab:robust_training}
\setlength{\tabcolsep}{4pt}
\begin{tabular*}{\linewidth}{@{\extracolsep{\fill}} ccccccccccccc}
\toprule
& \multicolumn{4}{c}{\textbf{Adv-trained ($\ell_2$ PGD, $r{=}0.5$)}}
& \multicolumn{8}{c}{\textbf{Gaussian aug ($\sigma{=}0.25$)}} \\
\cmidrule(lr){2-5}\cmidrule(lr){6-13}
& \multicolumn{2}{c}{$\detector_{Adv}$ (AT)}
& \multicolumn{2}{c}{$\detector_{lc}$ (AT)}
& \multicolumn{2}{c}{$\detector_{Adv}$}
& \multicolumn{2}{c}{$\detector_{lc}$}
& \multicolumn{2}{c}{RS (Cohen)}
& \multicolumn{2}{c}{RS (Conservative)} \\
\cmidrule(lr){2-3} \cmidrule(lr){4-5} \cmidrule(lr){6-7} \cmidrule(lr){8-9} \cmidrule(lr){10-11} \cmidrule(lr){12-13}
$\epsilon$ & TPR & FPR & TPR & FPR & TPR & FPR & TPR & FPR & TPR & FPR & TPR & FPR \\
\midrule
0.05 & 92.0 & 4.0  & 99.0 & 4.0  & 99.1  & 5.8  & 99.5  & 7.2  & 38.4 & 17.0 & 91.5  & 27.2 \\
0.1  & 94.9 & 7.0  & 99.2 & 7.6  & 99.6  & 11.2 & 100.0 & 13.8 & 47.6 & 21.4 & 93.8  & 30.8 \\
0.3  & 97.0 & 18.0 & 99.0 & 20.0 & 99.8  & 29.5 & 100.0 & 35.4 & 80.8 & 43.8 & 98.6  & 46.4 \\
0.5  & 98.4 & 29.8 & 99.6 & 32.2 & 100.0 & 43.9 & 100.0 & 51.5 & 95.8 & 65.0 & 100.0 & 65.8 \\
\bottomrule
\end{tabular*}
\end{table*}

A central implication of our framework is that robust-training advances that enlarge the underlying classifier's clean-margin distribution can improve the detector's TPR/FPR frontier. We confirm this empirically on two complementary robust-training primitives: $\ell_2$ PGD adversarial training and Gaussian-augmented training.

\paragraph{Adversarial training.}
We adversarially train a CIFAR-10 VGG16 classifier with $\ell_2$ PGD at radius $0.5$ and apply $\detector_{Adv}$ and $\detector_{lc}$ on top of the trained model.
Compared to the standard-trained baseline (Table~\ref{tab:cifar_compact}), FPR drops sharply at the same TPR --- e.g., $\detector_{lc}$'s FPR at $\epsilon{=}0.5$ falls from $93.6\%$ to $32.2\%$ while $\detector_{lc}$'s TPR remains $\geq 99\%$ across the reported $\epsilon$ on $\ell_2$ PGD (Table~\ref{tab:robust_training}).
Adversarial training therefore acts as a margin-enlarging primitive that translates directly into a better TPR/FPR Pareto frontier under our certificate.

\paragraph{Gaussian augmentation.}
A second standard primitive is training with Gaussian-augmented data ($\sigma{=}0.25$). Table~\ref{tab:robust_training} compares our detectors against RS (Cohen) and the conservative-RS variant on CIFAR-10 ($\ell_2$ PGD): noise augmentation lifts RS's TPR but its FPR stays substantially higher than ours at every reported $\epsilon$.

Full robust-training and Gaussian-augmentation results are in Table~\ref{tab:robust_training}.

Taken together, both settings make a single point: \revisetag{Comment-1}\textcolor{blue}{robust training is a compounding lever on our detection.} Stronger base-model margins improve the TPR/FPR frontier without changing the detector, so robustness gains translate directly into detection gains.

\textbf{Efficiency.}
Table~\ref{tab:runtime} reports per-sample detection time (mean $\pm$ std over $500$ inputs). $\detector_{Adv}$ is the fastest base detector (single PGD pass), $\detector_{lc}$ adds centroid optimisation on top, and both are markedly faster than RS (Cohen) at every scale. \textcolor{blue}{\revisetag{Review-A-3}The stronger $L'$ estimator increases latency by roughly $1.8$--$2.6\times$, but remains below RS on all three datasets; on ImageNet-1k, strong $\detector_{Adv}$ and strong $\detector_{lc}$ are still about $8.3\times$ and $4.9\times$ faster than RS, respectively.}

\section{Discussion}\label{sec:discussion}

\subsection{Three Levels of Certification}\label{sec:levels}

A natural concern with our framework is the gap between the formal certificate (Lemma~\ref{lemma.certified_detection}, Theorem~\ref{theorem.lowerbound_of_distance}) and the approximated detectors $\detector_{Adv}$, $\detector_{lc}$ that use the local estimator $L'$. We make this gap explicit by distinguishing three levels of certification, each with its own empirical signature in our results.

\textbf{Level 1 --- Fully certified.} On a network whose global Lipschitz constant $L_{\text{glob}}$ is exactly known (e.g., 1-Lipschitz SortNet~\cite{zhang_rethinking_2022}), Lemma~\ref{lemma.certified_detection} with $L = L_{\text{glob}}$ delivers a formal certificate \revisetag{Comment-1}\textcolor{blue}{at the declared radius}: every successful in-budget adversarial example is provably refused. Since $q(\xx)$ is unobserved at inference, we lower-bound it by $q(\xx)\ge 0$ (from clean correct classification), giving the conservative test $q(\adv) < c\,L_{\text{glob}}\,\|\delta\|$, which remains a sound certificate. Our SortNet evaluation (Table~\ref{tab:sortnet_certified}) realises this level, attaining $100\%$ TPR at FPR substantially below the baseline detectors built on SortNet's and CRM's own certified radii.

\textbf{Level 2 --- \revisetag{Comment-1}\textcolor{blue}{Conditionally certified} ($L' \geq L^*$).} On standard pretrained classifiers $L_{\text{glob}}$ is intractable, so we replace it with the local estimator $L'$. The certificate continues to hold for any specific successful in-budget adversarial example whenever the per-detector $L'$ condition holds on the relevant pair (Proposition~\ref{prop:approx_certifies}). \textcolor{blue}{We quantify this subset by cert\%, defined over all successful in-budget adversarial examples rather than only over rejected ones; on PyTorch pretrained ImageNet models, cert\% reaches $63$--$80\%$ for $\detector_{lc}$.} This number is comparable in magnitude to RS certified accuracy on the same backbone (cf.\ Cohen et al.~\cite{cohen2019certified}), but the two metrics measure different events and are not directly comparable.

\textbf{Level 3 --- \revisetag{Comment-1}\textcolor{blue}{Empirical protection}.} \textcolor{blue}{The difference between TPR and cert\% is the empirically rejected but not formally certified subset.} These examples are still protected by the deployed reject decision in our experiments (TPR $\geq 99\%$ on ImageNet), but we do not call them certified detections unless the per-detector $L'$ condition is verified. Our empirical TPR matches or exceeds dedicated empirical detectors (A\textsuperscript{2}D~\cite{zhao2024attack}, ContraNet~\cite{yang2022you}) on the same evaluation, so the empirical layer is competitive with SOTA empirical detection while remaining separate from the formal guarantee.

\subsection{\textcolor{blue}{FPR Growth with the Declared Radius}}\label{sec:fpr_eps}

\textcolor{blue}{FPR grows with $\epsilon$ because larger declared radii require wider rejection bands, mirroring the certified-robustness frontier. \revisetag{\mbox{Review-A-1}\tagsep\mbox{Review-B-2}\tagsep\mbox{Review-C-4}}Thus the defender selects the largest protection radius whose clean-FPR budget is acceptable, while cert\% reports which successful in-budget attacks are also formally certified. On SortNet at $\epsilon{=}0.5$, our fully certified detector keeps FPR at $35$--$41\%$, versus $86\%$/$96\%$ for SortNet/CRM certified-radius detectors; adversarial training cuts CIFAR-10 $\detector_{lc}$ FPR from $93.6\%$ to $32.2\%$.}

\subsection{Which detector certifies more, and when}\label{sec:detector_comparison}

The two approximated detectors do not dominate one another on cert\%; which one is tighter depends on the underlying classifier's clean-margin scale. On Lipschitz-constrained 1-Lipschitz SortNet (Table~\ref{tab:sortnet_certified}), $\detector_{Adv}$ certifies more: $58$--$72\%$ versus $37$--$58\%$ for $\detector_{lc}$. On unconstrained pretrained ImageNet classifiers (Table~\ref{tab:imagenet_cert_det}), the order reverses: \revisetag{Comment-1}\textcolor{blue}{$\detector_{lc}$ certifies $63$--$80\%$ versus $20$--$60\%$ for $\detector_{Adv}$.} Standard CIFAR-10 VGG16 falls in between, with the two within $1$--$4$ percentage points of each other (Table~\ref{tab:cifar_compact}).

\textcolor{blue}{cert\% tracks local Lipschitz stability rather than dataset complexity alone. \revisetag{\mbox{Review-B-1}\tagsep\mbox{Review-B-5}}If data complexity were the dominant driver, one would expect a sharp drop from CIFAR-10 to ImageNet-1k; instead the $\detector_{lc}$ certified fraction is comparable across the two settings, while larger variation appears across architectures on the same dataset and across constrained versus unconstrained backbones. Table~\ref{tab:multi_classifier} adds ResNet-18 and VGG-style CNN results that reproduce the same profile on an architecture variation for each small-scale dataset.}

The structural reason is that $L^*_{(\xx,\adv)}$ for $\detector_{Adv}$ scales with $q(\xx)+q(\adv)$, while $L^*_{(\lcadv,\xx)}$ for $\detector_{lc}$ scales with $q(\lcadv)+q(\xx)$. On a 1-Lipschitz network, $q(\lcadv)$ is bounded by the Lipschitz constant times the centroid travel and cannot escape the model's clean-margin scale; the simpler $\detector_{Adv}$ certificate (whose two terms live on the same short-distance pair) is in fact tighter. On deep unconstrained classifiers, $q(\xx)$ on a confident clean input is very large, $q(\lcadv)$ scales with it (the centroid moves through the network's high-margin region), and the centroid construction allows $L'$ to track $L^*$ rather than being dominated by the clean-margin term. In effect, the centroid's value as a high-margin anchor compounds with the model's own margin headroom; on Lipschitz-constrained networks there is no headroom to compound. We deploy the two detectors as a complementary pair: $\detector_{Adv}$ when a tight global Lipschitz constant is available, $\detector_{lc}$ on standard pretrained classifiers.

\subsection{Beyond $\ell_2$: $\ell_\infty$ Extension}\label{sec:lp_extension}

Lemma~\ref{lemma.certified_detection} extends to any $\ell_p$ norm with $c = 2^{(p-1)/p}$. We empirically tested an $\ell_\infty$ instantiation ($c{=}2$, $\ell_\infty$ PGD as the surrogate, $\|\cdot\|_\infty$ for distances) on standard CIFAR-10 VGG16: $\detector_{lc}$ reaches TPR $99.5\%$ at $\epsilon{=}1.5/255$ (FPR $80\%$) and $100\%$ at $\epsilon{=}2/255$ (FPR $90.5\%$). TPR confirms that our reduction transfers, but FPR is much higher than in the $\ell_2$ regime. The reason is purely norm-geometric: since $\|x\|_\infty \le \|x\|_2 \le \sqrt{n}\,\|x\|_\infty$, the $\ell_\infty$ Lipschitz constant of a network is substantially larger than its $\ell_2$ counterpart, making $L'$ approximation looser and the detection threshold more permissive. Tightening the $\ell_\infty$ certificate --- for instance via norm-specific Lipschitz estimators or per-dimension bounds --- is a natural direction we leave for future work.

\subsection{Practical Considerations and Limitations}\label{sec:practical}

We close by discussing three considerations that arise when deploying our framework, each with its own empirical or design-side mitigation.

\textbf{Coupling to the classifier's clean margin.} \textcolor{blue}{The detector's FPR is governed by the classifier's clean-margin distribution. \revisetag{Review-A-5}This follows from the Lipschitz-margin reduction: the margin distribution determines how tight the rejection band can be. Margin theory links larger normalized margins to robustness~\cite{bartlett_spectrally-normalized_2017,sokolic_robust_2017}; empirically, ViT-B/16 improves the ImageNet FPR profile over ResNet-50, and adversarial training cuts FPR sharply without changing the detector.}

\textbf{Single-pair $L'$ estimation.} Because $L'$ is computed from one perturbed pair $(z,z')$, Theorem~\ref{thm:deterministic} bounds the induced threshold gap. \textcolor{blue}{Under the revised denominator, pretrained-model cert\% remains $63$--$80\%$ for $\detector_{lc}$. \revisetag{\mbox{Review-A-3}\tagsep\mbox{Review-C-2}}Tables~\ref{tab:lprime_ratio} and~\ref{tab:bias_validation} (Appendix~\ref{appendix:full_experiments}) show that $L'/L^*$ concentrates near or above $1$ with mild violations; Table~\ref{tab:strong_estimator} shows the multi-probe cost/cert\% trade-off.} Hybrid pairwise or layer-wise estimators are compatible with our framework.

\textbf{Norm scope.} Our primary focus is $\ell_2$, which already covers many security-relevant adversarial threats --- physical-world stickers on traffic signs, frequency-domain perturbations, low-energy patches --- via principled $\ell_2$ formulations. The $\ell_\infty$ extension is sketched in Section~\ref{sec:lp_extension}.

\textbf{Attention-based architectures.} \textcolor{blue}{We do not rely on a global Lipschitz bound, which can be loose or unbounded for softmax attention outside a bounded input region~\cite{kim2021lipschitz}. \revisetag{Review-B-4}Our certificate needs only the local pairwise constant $L^*$ on the bounded ball $B(\xx,\epsilon)$, where both $L^*$ and $L'$ are finite; the guarantee on attention models is therefore conditional on $L'\ge L^*$ in the same way as other unconstrained pretrained classifiers. The ViT-B/16 results in Table~\ref{tab:vit_a2d} support this scope, with cert\% at least $63\%$ alongside near-perfect empirical TPR.}

\textbf{Beyond image classification.} \textcolor{blue}{The reduction requires a scalar score margin and a norm-bounded perturbation model, not an image-specific architecture. \revisetag{Review-C-1}Extending the same certificate to detection, segmentation, or multimodal pipelines is therefore conceptually feasible whenever an appropriate task margin can be defined, but certified robustness for those tasks is less mature; we leave such extensions to future work.}

% \textbf{Reliance on a host-pipeline fallback.} Like all detection-then-refuse defences, our system relies on the host pipeline having a meaningful fallback for refused inputs (human review, safe-stop, additional verification factor). The deployments we target (Section~\ref{sec:threat_model}) all have natural fallbacks built in: an AV pipeline can defer to safe-stop, content moderation can route flagged uploads to human reviewers, and biometric authentication can request a second factor.

\section{Conclusion}\label{sec:conclusion}

We propose a detection-oriented perspective on defending DNNs against adversarial examples: rather than certifying that predictions remain unchanged, we certify that adversarial inputs are detected \revisetag{Comment-1}\textcolor{blue}{when the stated Lipschitz conditions hold}. We prove via Lipschitz-margin analysis that certified $\ell_2$ robustness implies certified adversarial detection, and instantiate two practical detectors that operate on any pretrained classifier without retraining or architectural change. Across MNIST, CIFAR-10, and ImageNet-1k on four backbones, the framework matches or outperforms five baselines (CTT, SortNet, CRM, ContraNet, A$^2$D), \textcolor{blue}{empirically rejects nearly all measured attacks, and formally certifies $63$--$80\%$ of successful in-budget ImageNet adversarial examples under the revised cert\% metric.} Furthermore, \textcolor{blue}{robust training can improve the operating point of the same detectors by enlarging margins.}
% Practical considerations and limitations are discussed in Section~\ref{sec:practical}.

\section*{Acknowledgments}

We thank the anonymous reviewers for their constructive feedback and suggestions for improving this work. This work was 
supported by the Open Research Fund of the State Key Laboratory 
of Blockchain and Data Security, Zhejiang University, the Natural Science Foundation of China (No. U25A20427, U25A20425, 62602259, and 12601509), the Startup Foundation for Introducing Talent of NUIST (No. 1513142601187), Basic Research Program of Jiangsu(Grants No. SBK20260406892), and the supporting funds for talents of Nanjing University of Aeronautics and Astronautics.

\bibliographystyle{IEEEtran}
\bibliography{ndss/3reference}

\clearpage
\appendix
% \section{Related}
% OOD detection, robust OOD detection aims at detect out of ditribution samples while making the models robust against adversarial perturbations, i.e. small modifications of OOD inputs are still detectable as OOD samples.
\subsection{Ethical Considerations}
\label{appendix:ethics}
This work develops an adversarial example detection method that can be deployed in situations like Autonomous-driving perception, content moderation, and access control. While our detectors are shown to have almost 100\% recall for limit perturbations in $\ell_2-norm$,
we acknowledge that there exists fasle negatives when the adversary have large perturbation budget in reality. However, the service provider in our situation already have a meaningful build-in fallback for refused inputs .

\subsection{Notation Reference}\label{appendix:symbols}

Table~\ref{tab:notation} summarizes the symbols used in this paper.

\begin{table}[t]
\centering
\small
\caption{Notation summary.}
\label{tab:notation}
\setlength{\tabcolsep}{4pt}
\renewcommand{\arraystretch}{1.1}
\begin{tabular*}{\linewidth}{@{\extracolsep{\fill}} l l}
\toprule
\textbf{Symbol} & \textbf{Meaning} \\
\midrule
\multicolumn{2}{l}{\textit{Inputs and perturbations}} \\
$\xx$        & Clean input. \\
$y$          & True class of $\xx$. \\
$\delta$     & Adversarial perturbation. \\
$\epsilon$   & Perturbation budget. \\
$\adv$       & Adversarial example, $\xx{+}\delta$. \\
$z$          & Inference-time test input. \\
$z'$         & Defender's perturbation of $z$. \\
\midrule
\multicolumn{2}{l}{\textit{Network and margin}} \\
$f$              & DNN classifier. \\
$f_i(\xx)$       & $i$-th class score. \\
$q(u;k)$         & Class-$k$ margin: $f_k(u)-\max_{j\ne k}f_j(u)$. \\
$q(u)$           & Shorthand for $q(u;\arg\max_k f_k(u))$ (own-label, $\ge 0$). \\
$\hat y$         & Input pseudo-label $\arg\max f(z)$, fixed reference. \\
$\lcx$           & Local centroid of $\xx$. \\
$\lcadv$         & Local centroid of $\adv$. \\
\midrule
\multicolumn{2}{l}{\textit{Lipschitz quantities}} \\
$\|\cdot\|_p$         & $\ell_p$ norm; default $p{=}2$. \\
$c$                   & Norm constant $2^{(p-1)/p}$. \\
$L$                   & Lipschitz constant of $f$. \\
$L_{\text{glob}}$     & Global Lipschitz constant. \\
$L'$                  & Local Lipschitz estimator. \\
$L^*$                 & Optimal pairwise Lipschitz. \\
$L^*_{(\xx,\adv)}$    & Pairwise on $(\xx,\adv)$. \\
$L^*_{(\lcadv,\xx)}$  & Pairwise on (centroid, clean). \\
\midrule
\multicolumn{2}{l}{\textit{Detectors and metrics}} \\
$\detector$           & Detector, $\{0,1\}$ output. \\
$\detector_{Adv}$     & Via Claim~\ref{claim:approximate_adv}. \\
$\detector_{lc}$      & Via Claim~\ref{claim:approximate_lc}. \\
TPR                   & Adversarial-detection rate. \\
FPR                   & Clean-rejection rate. \\
cert\%                & \textcolor{blue}{Adversarials both rejected and formally certified.} \\
\bottomrule
\end{tabular*}
\end{table}

% \subsection{Attacks Settings}
% We use BIM with 10 steps, PGD with 40 steps and 10 random restarts, and CW-$\ell_2$ with 40 steps.
% All attacks are the $\ell_2$ version, or projected to the $\ell_2$ ball of radius $\|\delta\|$.

\ifshowmainalgorithms
\else
\subsection{Algorithm Details}\label{appendix:algorithms}

This subsection gives the implementation-level pseudocode referenced in Section~\ref{sec:method}: Algorithm~\ref{alg:perturbation} generates the defender-side perturbation, Algorithm~\ref{alg:centroid} computes the local centroid, and Algorithms~\ref{alg:detector-lemma1}--\ref{alg:detector-adv} instantiate $\detector_{Adv}$ and $\detector_{lc}$.

\fi

\subsection{Detailed Results across Datasets}
\label{appendix:full_experiments}

This subsection reports the per-$\epsilon$, per-attack and per-noise-level breakdowns that did not fit in the body. \revisetag{\mbox{Review-A-3}\tagsep\mbox{Review-C-2}}\textcolor{blue}{It also includes the $\Delta_T$ bias-validation table (Table~\ref{tab:bias_validation}).} The trends parallel the main-text summary: $\detector_{lc}$ achieves near-perfect TPR across all settings, $\detector_{Adv}$ trades a small TPR drop for substantially lower FPR, and the Cohen detector incurs FPR that escalates rapidly with $\epsilon$.

Tables~\ref{tab:mnist-tpr}--\ref{tab:cifar-clean-fpr} present per-attack TPR and clean FPR on MNIST and CIFAR-10 under noise levels $\sigma \in \{0.125, 0.25\}$. Table~\ref{tab:imagenet_cert_det_full} extends the body's ImageNet-1k results (Table~\ref{tab:imagenet_cert_det}) to the full $\epsilon \in [0.1, 1.0]$ range at step $0.1$.

Table~\ref{tab:bias_validation} reports the SortNet bias-bound check used to validate the deterministic $\Delta_T$ envelope.

\begin{table}[t]
\centering
\small
\caption{\textbf{\textcolor{blue}{Empirical tightness of $\Delta_T$ on SortNet.}}
\textcolor{blue}{Median actual $\Delta_T$ and its share of the per-sample theoretical envelope from \eqref{eq:det-eta}, computed on detected adversarials with $L'\le L_{\text{glob}}$. The remaining samples satisfy $L'\ge L_{\text{glob}}\ge L^*$ and are trivially certified.}}
\label{tab:bias_validation}
\setlength{\tabcolsep}{4pt}
\begin{tabular*}{\linewidth}{@{\extracolsep{\fill}} ccc}
\toprule
$\epsilon$ & $\detector_{Adv}$ & $\detector_{lc}$ \\
\midrule
0.3 & 0.31 (74\%) & 0.20 (16\%) \\
0.5 & 0.51 (72\%) & 0.29 (19\%) \\
1.0 & 1.02 (63\%) & 0.85 (44\%) \\
\bottomrule
\end{tabular*}
\end{table}

\begin{table*}[t]
\centering
\caption{Comparison of our approximated detectors with CTT~\cite{shumailov2020towards} on MNIST and CIFAR-10.
We report true positive rate (TPR, \%) and false positive rate (FPR, \%) under different attacks and perturbation levels.
CTT-loose and CTT-strict correspond to two thresholding variants in~\cite{shumailov2020towards}.}
\vspace{-0.05in}
\label{tab:ctt_comparison}
\setlength{\tabcolsep}{5pt}
\begin{tabular*}{\linewidth}{@{\extracolsep{\fill}} llcccccccccc}
\toprule
\multirow{2}{*}{Dataset} & \multirow{2}{*}{Attack} & \multirow{2}{*}{$\epsilon$} &
\multicolumn{2}{c}{$\detector_{Adv}$ (Ours)} &
\multicolumn{2}{c}{$\detector_{lc}$ (Ours)} &
\multicolumn{2}{c}{CTT-loose} &
\multicolumn{2}{c}{CTT-strict} \\
\cmidrule(lr){4-5} \cmidrule(lr){6-7} \cmidrule(lr){8-9} \cmidrule(lr){10-11}
 &  &  & TPR & FPR & TPR & FPR & TPR & FPR & TPR & FPR \\
\midrule
\multirow{4}{*}{MNIST}
 & FGSM & 0.1 & 100.0 & 0.5 & 100.0 & 1.0 & 100.0 & 1.6 & 100.0 & 1.1 \\
 & FGSM & 0.2 & 100.0 & 1.5 & 100.0 & 3.0 & 100.0 & 1.6 & 100.0 & 1.1 \\
 & PGD  & 0.1 & 100.0 & 0.5 & 100.0 & 1.0 & 100.0 & 1.6 & 100.0 & 1.1 \\
 & PGD  & 0.2 & 100.0 & 1.5 & 100.0 & 3.0 & 100.0 & 1.6 & 100.0 & 1.1 \\
\midrule
\multirow{4}{*}{CIFAR-10}
 & FGSM & 0.02 & 95.0 & 2.5 & 100.0 & 6.0 & 95.7 & 6.4 & 96.5 & 3.0 \\
 & FGSM & 0.04 & 93.0 & 5.0 & 100.0 & 13.0 & 93.6 & 6.4 & 93.1 & 3.0 \\
 & PGD  & 0.01 & 89.0 & 1.0 & 100.0 & 3.0 & 24.3 & 6.4 & 16.6 & 3.0 \\
 & PGD  & 0.02 & 96.0 & 2.5 & 100.0 & 6.0 & 70.3 & 6.4 & 49.9 & 3.0 \\
\bottomrule
\end{tabular*}
\end{table*}

\begin{table*}[t]
\centering
\scriptsize
\caption{TPR (\%) of different attacks under varying $\ell_2$ budgets $\epsilon$.
For the Cohen detector~\cite{cohen2019certified}, we report certified TPR against \textbf{BIM} under two noise levels.}
\vspace{-0.05in}
\label{tab:mnist-tpr}
\label{tab:cifar-tpr}
\setlength{\tabcolsep}{4pt}
\begin{tabular*}{\linewidth}{@{\extracolsep{\fill}} llcccccccccc}
\toprule
\multirow{2}{*}{Dataset} & \multirow{2}{*}{$\epsilon$}
& \multicolumn{4}{c}{$\detector_{Adv}$}
& \multicolumn{4}{c}{$\detector_{lc}$}
& \multicolumn{2}{c}{Cohen} \\
\cmidrule(lr){3-6} \cmidrule(lr){7-10} \cmidrule(lr){11-12}
& & FGSM & PGD & BIM & CW & FGSM & PGD & BIM & CW & $\sigma=0.125$ & $\sigma=0.25$ \\
\midrule
\multirow{9}{*}{MNIST}
& 0.05 & 100.0 & 100.0 & 100.0 & 100.0 & 100.0 & 100.0 & 100.0 & 100.0 & 60.0 & 33.0 \\
& 0.10 & 100.0 & 100.0 & 100.0 & 100.0 & 100.0 & 100.0 & 100.0 & 100.0 & 66.7 & 44.5 \\
& 0.30 & 100.0 & 100.0 & 100.0 & 100.0 & 100.0 & 100.0 & 100.0 & 100.0 & 100.0 & 63.9 \\
& 0.50 & 100.0 & 100.0 & 100.0 & 100.0 & 100.0 & 100.0 & 100.0 & 100.0 & 100.0 & 100.0 \\
& 1.00 & 100.0 & 100.0 & 100.0 & 100.0 & 100.0 & 100.0 & 100.0 & 100.0 & 100.0 & 100.0 \\
& 1.50 & 100.0 & 100.0 & 100.0 & 100.0 & 100.0 & 100.0 & 100.0 & 100.0 & 100.0 & 100.0 \\
& 2.00 & 100.0 & 100.0 & 100.0 & 100.0 & 100.0 & 100.0 & 100.0 & 100.0 & 100.0 & 100.0 \\
& 2.50 & 100.0 & 100.0 & 100.0 & 100.0 & 100.0 & 100.0 & 100.0 & 100.0 & 100.0 & 100.0 \\
& 3.00 & 100.0 & 100.0 & 100.0 & 100.0 & 100.0 & 100.0 & 100.0 & 100.0 & 100.0 & 100.0 \\
\midrule
\multirow{8}{*}{CIFAR-10}
& 0.05 &  99.8 & 100.0 & 100.0 &  98.8 &  99.8 & 100.0 & 100.0 &  99.8 & 67.0 & 38.5 \\
& 0.10 & 100.0 & 100.0 & 100.0 &  99.6 & 100.0 & 100.0 & 100.0 &  99.8 & 74.5 & 53.0 \\
& 0.20 & 100.0 & 100.0 & 100.0 &  99.8 & 100.0 & 100.0 & 100.0 &  99.8 & 89.4 & 70.5 \\
& 0.30 & 100.0 & 100.0 & 100.0 & 100.0 & 100.0 & 100.0 & 100.0 & 100.0 & 100.0 & 95.8 \\
& 0.40 & 100.0 & 100.0 & 100.0 & 100.0 & 100.0 & 100.0 & 100.0 & 100.0 & 100.0 & 100.0 \\
& 0.50 & 100.0 & 100.0 & 100.0 & 100.0 & 100.0 & 100.0 & 100.0 & 100.0 & 100.0 & 100.0 \\
& 0.75 & 100.0 & 100.0 & 100.0 & 100.0 & 100.0 & 100.0 & 100.0 & 100.0 & 100.0 & 100.0 \\
& 1.00 & 100.0 & 100.0 & 100.0 & 100.0 & 100.0 & 100.0 & 100.0 & 100.0 & 100.0 & 100.0 \\
\bottomrule
\end{tabular*}
\vspace{-0.05in}
\end{table*}

\begin{table}[t]
\centering
\small
\caption{Clean FPR (\%) of our detectors vs.\ RS (Cohen) across perturbation budgets. Lower is better.}
\label{tab:clean_fpr_side_by_side}
\label{tab:mnist-clean-fpr}
\label{tab:cifar-clean-fpr}
\setlength{\tabcolsep}{4pt}
\begin{tabular}{llrrrr}
\toprule
\multirow{2}{*}{Dataset} & \multirow{2}{*}{$\epsilon$} & \multicolumn{2}{c}{Ours} & \multicolumn{2}{c}{Cohen} \\
\cmidrule(lr){3-4}\cmidrule(lr){5-6}
& & $\detector_{Adv}$ & $\detector_{lc}$ & $\sigma{=}0.125$ & $\sigma{=}0.25$ \\
\midrule
\multirow{9}{*}{MNIST}
& 0.05 & 0.0  & 0.0  & 0.5   & 2.5 \\
& 0.10 & 0.5  & 1.0  & 1.5   & 4.0 \\
& 0.30 & 5.0  & 5.5  & 100.0 & 8.5 \\
& 0.50 & 8.0  & 8.5  & 100.0 & 22.0 \\
& 1.00 & 17.0 & 17.5 & 100.0 & 98.0 \\
& 1.50 & 25.5 & 28.0 & 100.0 & 98.5 \\
& 2.00 & 47.0 & 47.5 & 100.0 & 98.0 \\
& 2.50 & 68.0 & 68.5 & 100.0 & 98.5 \\
& 3.00 & 81.0 & 80.5 & 100.0 & 98.5 \\
\midrule
\multirow{8}{*}{CIFAR-10}
& 0.05 & 15.3 & 15.6 & 35.0 & 46.0 \\
& 0.10 & 29.4 & 29.9 & 40.5 & 60.2 \\
& 0.20 & 55.4 & 56.1 & 68.1 & 68.8 \\
& 0.30 & 74.7 & 75.7 & 100.0 & 96.0 \\
& 0.40 & 86.0 & 87.1 & 100.0 & 99.8 \\
& 0.50 & 92.7 & 93.5 & 100.0 & 99.8 \\
& 0.75 & 98.3 & 98.7 & 100.0 & 100.0 \\
& 1.00 & 99.6 & 99.7 & 100.0 & 100.0 \\
\bottomrule
\end{tabular}
\end{table}

\begin{table}[t]
\centering
\scriptsize
\caption{\textbf{\textcolor{blue}{Architecture variation on CIFAR-10 and MNIST.}}
\textcolor{blue}{BIM attacks with $n_{\mathrm{adv}}{=}500$ and $n_{\mathrm{clean}}{=}1000$. We evaluate CIFAR-10 on ResNet-18 and MNIST on a VGG-style CNN to check that the detection profile is not tied to the primary classifiers.}}
\label{tab:multi_classifier}
\setlength{\tabcolsep}{5pt}
\begin{tabular*}{\linewidth}{@{\extracolsep{\fill}} llcccccc}
\toprule
\multirow{2}{*}{Dataset} & \multirow{2}{*}{Backbone} & \multirow{2}{*}{$\epsilon$} &
\multirow{2}{*}{TPR} & \multicolumn{2}{c}{FPR} & \multicolumn{2}{c}{cert\%} \\
\cmidrule(lr){5-6} \cmidrule(lr){7-8}
 & & & & $\detector_{Adv}$ & $\detector_{lc}$ & $\detector_{Adv}$ & $\detector_{lc}$ \\
\midrule
\multirow{4}{*}{CIFAR-10} & \multirow{4}{*}{ResNet-18} & 0.05 & 100.0 & 12.2 & 12.2 & 70.0 & 71.3 \\
 & & 0.10 & 100.0 & 29.7 & 29.7 & 86.0 & 85.3 \\
 & & 0.30 & 100.0 & 74.9 & 75.1 & 78.7 & 77.0 \\
 & & 0.50 & 100.0 & 90.5 & 91.0 & 71.7 & 67.7 \\
\midrule
\multirow{3}{*}{MNIST} & \multirow{3}{*}{VGG-style CNN} & 0.50 & 100.0 & 2.4 & 3.6 & 95.0 & 100.0 \\
 & & 1.00 & 100.0 & 7.5 & 16.3 & 84.7 & 96.3 \\
 & & 1.50 & 100.0 & 26.1 & 54.2 & 73.0 & 90.7 \\
\bottomrule
\end{tabular*}
\vspace{2pt}
\end{table}

\begin{figure*}[t]
\centering
% Row 1: D_Adv across MNIST / CIFAR-10 / ImageNet-1k
\begin{minipage}{0.32\textwidth}
    \centering
    \includegraphics[width=\linewidth]{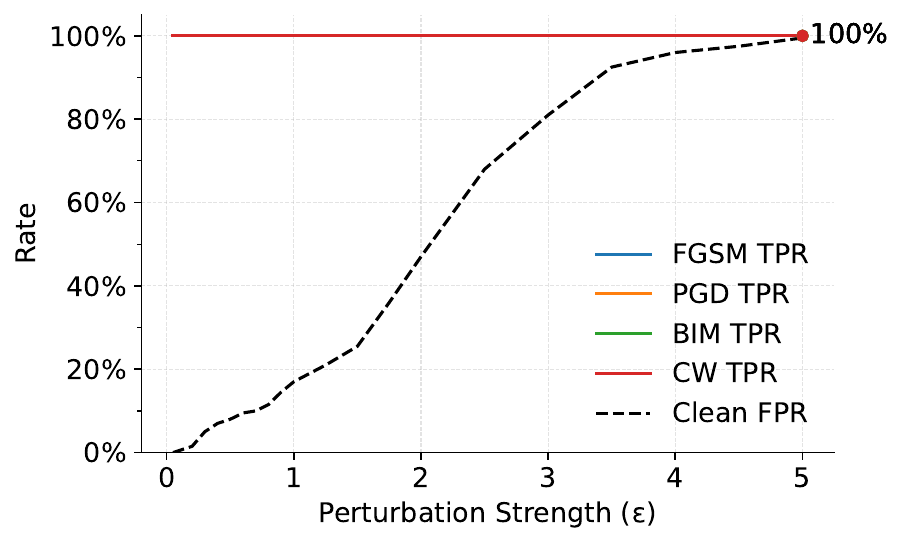}
    \caption*{(a) $\detector_{Adv}$ on MNIST}
\end{minipage}\hfill
\begin{minipage}{0.32\textwidth}
    \centering
    \includegraphics[width=\linewidth]{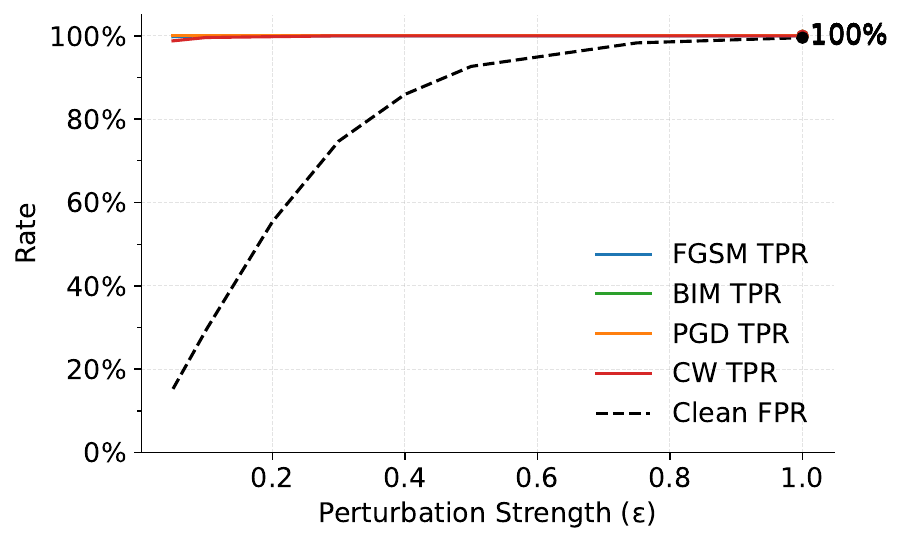}
    \caption*{(b) $\detector_{Adv}$ on CIFAR-10}
\end{minipage}\hfill
\begin{minipage}{0.32\textwidth}
    \centering
    \includegraphics[width=\linewidth]{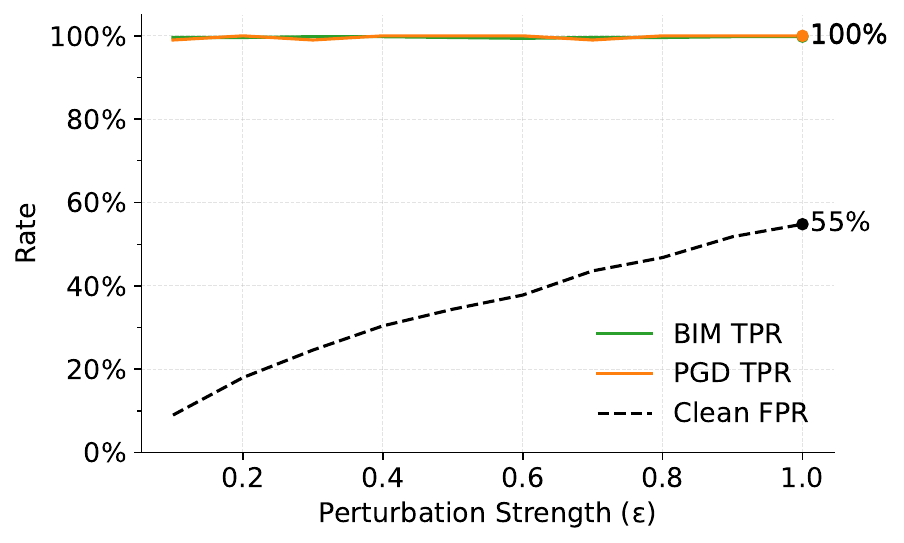}
    \caption*{(c) $\detector_{Adv}$ on ImageNet-1k}
\end{minipage}

\vspace{0.5em}

% Row 2: D_lc across MNIST / CIFAR-10 / ImageNet-1k
\begin{minipage}{0.32\textwidth}
    \centering
    \includegraphics[width=\linewidth]{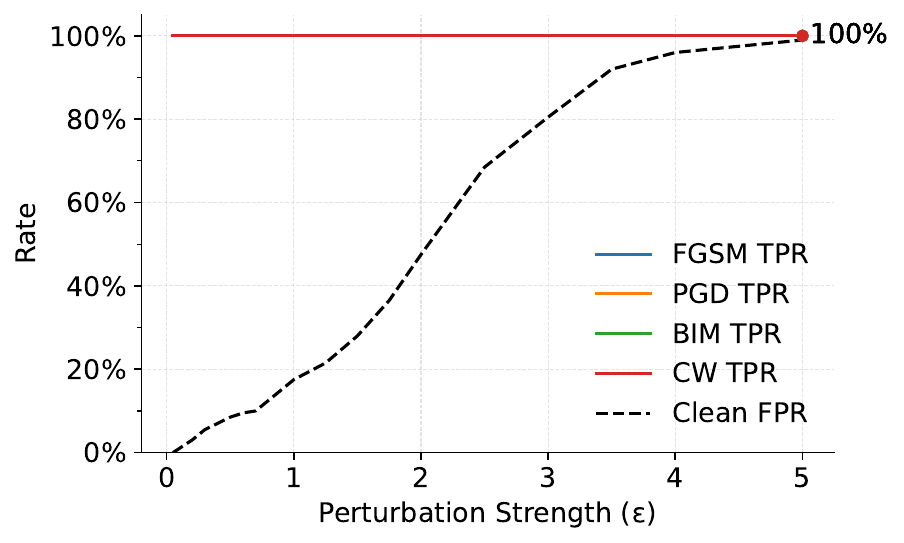}
    \caption*{(d) $\detector_{lc}$ on MNIST}
\end{minipage}\hfill
\begin{minipage}{0.32\textwidth}
    \centering
    \includegraphics[width=\linewidth]{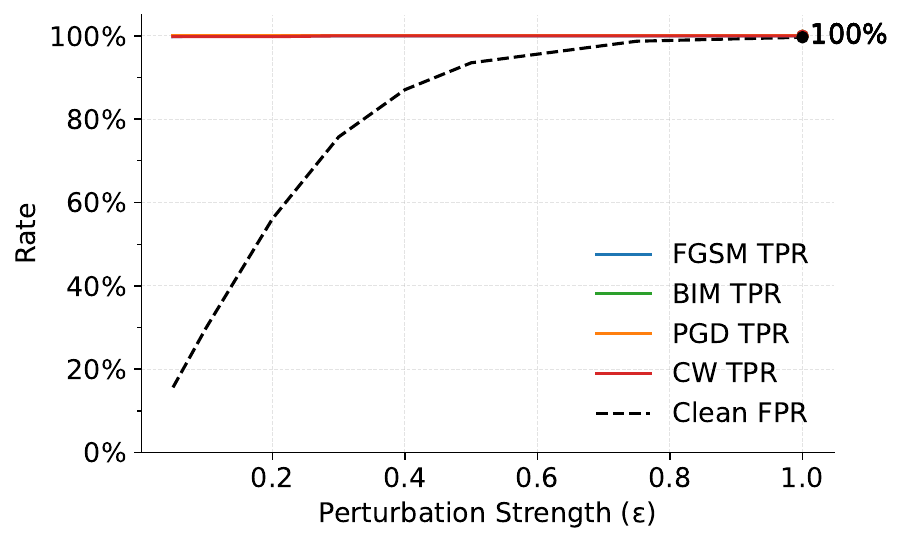}
    \caption*{(e) $\detector_{lc}$ on CIFAR-10}
\end{minipage}\hfill
\begin{minipage}{0.32\textwidth}
    \centering
    \includegraphics[width=\linewidth]{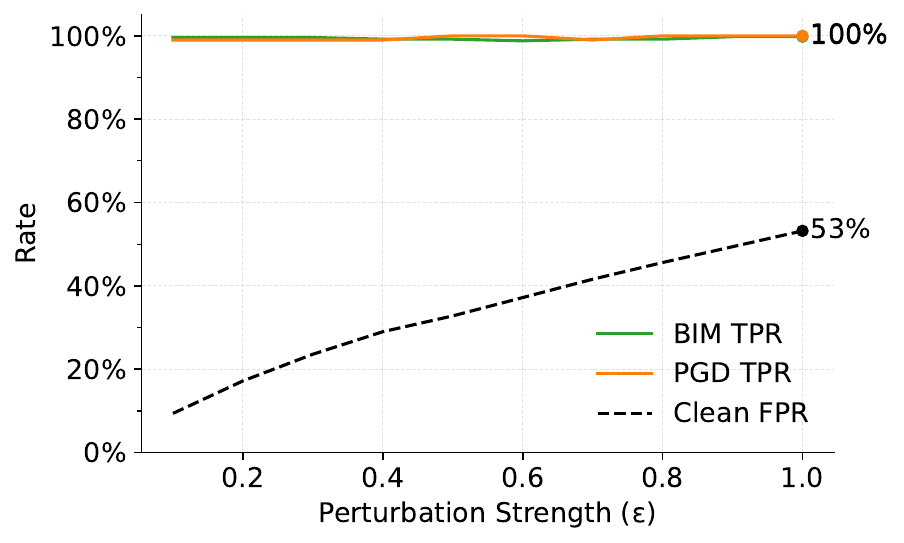}
    \caption*{(f) $\detector_{lc}$ on ImageNet-1k}
\end{minipage}
\caption{Summary curves: TPR on adversarial inputs and FPR on clean inputs across $\ell_2$ perturbation magnitudes $\epsilon$, for $\detector_{Adv}$ (top row) and $\detector_{lc}$ (bottom row). Columns sweep MNIST (LeNet-5), CIFAR-10 (VGG16), and ImageNet-1k (ResNet-50). }
\label{fig:summary_curves}
\end{figure*}

\begin{figure}[t]
    \centering
    \includegraphics[width=0.9\linewidth]{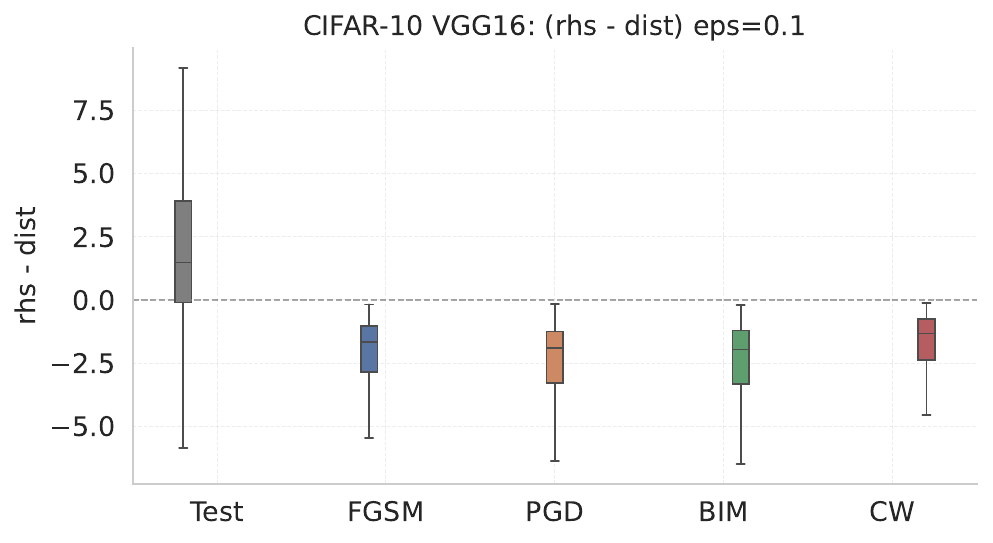}
    \caption{
    Distribution of $(\mathrm{RHS} - d_z)$ values under different attacks. 
    Negative values ($\mathrm{RHS} < d_z$) indicate samples flagged as adversarial.
    }
    \label{fig:rhs_dist}
    \vspace{-0.1in}
\end{figure}

\begin{table}[t]
\centering
\small
\caption{\textbf{\textcolor{blue}{Cost--tightness sweep for the stronger estimator.}}
\textcolor{blue}{CIFAR-10 VGG16, $\detector_{Adv}$, $\epsilon{=}0.1$. }}
\label{tab:strong_estimator_cost}
\begin{tabular}{lcc}
\toprule
Estimator & cert\% & ms/sample \\
\midrule
Baseline & 76.4 & 22 \\
$K{=}1,T{=}3$ & 91.6 & 108 \\
$K{=}1,T{=}6$ & 93.5 & 201 \\
$K{=}1,T{=}9$ & 93.7 & 293 \\
$K{=}1,T{=}12$ & 93.7 & 383 \\
$K{=}2,T{=}12$ & 95.4 & 563 \\
\bottomrule
\end{tabular}
\end{table}

\begin{table}[t]
\centering
\scriptsize
\caption{\textbf{\textcolor{blue}{Empirical distribution of $L'/L^*$.}}
\textcolor{blue}{CIFAR-10 VGG16, BIM, $n{=}500$. Violations are samples with $L'/L^*<1$; the last column reports their ratio distribution. $\detector_{lc}$ has a right-skewed tail at $\epsilon{=}0.10,0.50$ from a few samples with tiny $L^*$; its median remains near $1$.}}
\label{tab:lprime_ratio}
\setlength{\tabcolsep}{5pt}
\begin{tabular*}{\linewidth}{@{\extracolsep{\fill}} lccccc}
\toprule
\multicolumn{6}{c}{\textbf{$\detector_{Adv}$}} \\
\cmidrule(lr){1-6}
$\epsilon$ & cert\% & $L'/L^*$ mean$\pm$std & Median & Viol. \% & Viol. mean$\pm$std \\
\midrule
0.05 & 62.6 & $1.037\pm0.20$ & 1.007 & 37.4 & $0.949\pm0.082$ \\
0.10 & 76.6 & $1.067\pm0.24$ & 1.034 & 23.4 & $0.908\pm0.107$ \\
0.30 & 83.0 & $1.153\pm0.29$ & 1.118 & 17.0 & $0.855\pm0.133$ \\
0.50 & 82.4 & $1.194\pm0.30$ & 1.162 & 17.6 & $0.861\pm0.130$ \\
\midrule
\multicolumn{6}{c}{\textbf{$\detector_{lc}$}} \\
\cmidrule(lr){1-6}
$\epsilon$ & cert\% & $L'/L^*$ mean$\pm$std & Median & Viol. \% & Viol. mean$\pm$std \\
\midrule
0.05 & 62.0 & $0.985\pm0.13$ & 1.001 & 38.0 & $0.947\pm0.200$ \\
0.10 & 75.2 & $1.123\pm2.63$ & 1.006 & 24.8 & $0.952\pm0.132$ \\
0.30 & 80.4 & $1.064\pm0.14$ & 1.039 & 19.6 & $0.939\pm0.065$ \\
0.50 & 78.8 & $1.167\pm1.40$ & 1.060 & 21.2 & $0.924\pm0.081$ \\
\bottomrule
\end{tabular*}
\vspace{2pt}
\end{table}

\begin{table}[t]
\centering
\caption{Full ImageNet-1k results (PyTorch pretrained ResNet-50) at step $\epsilon = 0.1$. TPR on $\ell_2$ BIM attacks; FPR on clean samples.}
\label{tab:imagenet_cert_det_full}

\setlength{\tabcolsep}{4pt}
\begin{tabular*}{\linewidth}{@{\extracolsep{\fill}} ccccccc}
\toprule
$\epsilon$
& \multicolumn{2}{c}{$\detector_{Adv}$ (\%)}
& \multicolumn{2}{c}{$\detector_{lc}$ (\%)}
& \multicolumn{2}{c}{RS (Cohen) (\%)} \\
\cmidrule(lr){2-3} \cmidrule(lr){4-5} \cmidrule(lr){6-7}
& TPR $\uparrow$ & FPR $\downarrow$
& TPR $\uparrow$ & FPR $\downarrow$
& TPR $\uparrow$ & FPR $\downarrow$ \\
\midrule
0.1 & 99.5 &  9.0 & 99.7 &  9.4 & 34.0 & 17.0 \\
0.2 & 99.6 & 18.0 & 99.6 & 17.2 & 53.0 & 30.0 \\
0.3 & 99.8 & 24.6 & 99.6 & 23.7 & 62.5 & 36.0 \\
0.4 & 99.9 & 30.3 & 99.2 & 29.0 & 82.0 & 68.0 \\
0.5 & 99.6 & 34.4 & 99.3 & 32.8 & 85.5 & 51.0 \\
0.6 & 99.4 & 37.8 & 98.8 & 37.2 & 100.0 & 100.0 \\
0.7 & 99.6 & 43.5 & 99.5 & 41.6 & 100.0 & 100.0 \\
0.8 & 99.3 & 46.7 & 99.2 & 45.3 & 100.0 & 100.0 \\
0.9 & 99.7 & 51.9 & 99.8 & 49.4 & 100.0 & 100.0 \\
1.0 & 99.8 & 54.8 & 99.8 & 53.5 & 100.0 & 100.0 \\
\bottomrule
\end{tabular*}
\end{table}

Table~\ref{tab:multi_classifier} adds the architecture variations on small-scale datasets; Table~\ref{tab:ctt_comparison} reports the complementary CTT comparison omitted from the body for space.

Figures~\ref{fig:summary_curves} and~\ref{fig:rhs_dist} provide the full perturbation-scale curves and per-sample detection-statistic distributions summarized in Section~\ref{sec:experiments}.

Table~\ref{tab:lprime_ratio} and Table~\ref{tab:strong_estimator_cost} give the supporting diagnostics for the revised certified-coverage analysis: the realised $L'/L^*$ distribution and the cost--tightness sweep for stronger $L'$ estimation.

On \textbf{MNIST}, both $\detector_{Adv}$ and $\detector_{lc}$ attain $100\%$ TPR across the full range; FPR stays below $3\%$ for $\epsilon \le 0.2$ and grows gradually thereafter (about $17\%$ at $\epsilon{=}1.0$, $80\%$ at $\epsilon{=}3.0$).

On \textbf{CIFAR-10}, the difference is more pronounced. $\detector_{Adv}$ achieves $69$--$96\%$ TPR with much lower FPR ($7.5$--$70\%$); $\detector_{lc}$ reaches near-perfect TPR ($\approx 100\%$) at the cost of higher FPR ($16$--$100\%$) for large $\epsilon$.

On \textbf{ImageNet-1k}, both our detectors stay $\geq 98.8\%$ TPR throughout, while keeping FPR under $55\%$ even at $\epsilon{=}1.0$. Cohen exhibits FPR exceeding $50\%$ once $\epsilon \ge 0.4$ and saturates at $100\%$ for $\epsilon \ge 0.6$.

\ifcameraready
\else

\subsection{Proof of Lemma~\ref{lemma.certified_detection} and Theorem~\ref{theorem.lowerbound_of_distance}}
\label{proof.adv_margin_bound}
% We first present the following lemma which will be used to prove Lemma~\ref{lemma.2}
\begin{lemma}\label{lemma.rethinking}
(Proposition 3.1 in ~\cite{zhang_rethinking_2022})For a neural network $f$ with $Lipschitz$ constant $L$ under $\ell_p$ -norm $||\cdot||_p$, $f$ is provably robust against perturbations $||\delta||_p<\frac{1}{cL}margin(f(\xx))$, i.e.
\begin{equation}
argmax f(\xx) = argmax f(\xx+\delta),\ \forall||\delta||_p<\frac{1}{cL}margin(f(\xx)),
\end{equation}
where $c=2^{\frac{p-1}{p}}$ is the constant depending only on the norm $||\cdot||_p$, which is $2$ for $\ell_\infty$ (and $\sqrt{2}$ for $\ell_2$).
\end{lemma}

The proof for Lemma~\ref{lemma.certified_detection} is as follows.
\begin{proof}
    Let $y,j$ be the class of $\xx$ and $\adv$ predicted by $f$,
    Then we have
    \begin{equation}\label{equ.10}
        \begin{aligned}
            ||f(\adv)-f(\xx)||_p&= (|(f_y(\adv)-f_y(\xx)|^p+\sum_{i\neq y}|f_i(\adv)-f_i(\xx)|^p)^{\frac{1}{p}}\\
            &\geq ( |(f_y(\adv)-f_y(\xx)|^p +|f_j(\adv)-f_j(\xx)|^p)^{\frac{1}{p}}.
        \end{aligned}
    \end{equation}
    Let $dz_y = f_y(\adv)-f_y(\xx), dz_j=f_j(\adv)-f_j(\xx)$.
    Considering a polynomial $g(r) = r^p +(dz_j-dz_y-r)^p$, then 
    \begin{equation*}
    \begin{aligned}
        g(-dz_y) &=(-dz_y)^p+dz_j^p \\
        &= (-f_y(\adv)+f_y(\xx))^p +(f_j(\adv)-f_j(\xx))^p \\
        &\leq||f(\adv)-f(\xx)||_p^p,
    \end{aligned}
    \end{equation*} 
    where the last  line inequality is caused by \eqref{equ.10}.
    By analyzing the derivatives of $g(z)$, we can find that $g(z)$ has a global minimum at $r=\frac{dz_j-dz_y}{2}$, using the fact that $f_j(\adv)> f_y(\adv)$ and $f_j(\xx)\leq f_y(\xx)$.Then, 
    \begin{equation}\label{equ.12}
    \begin{aligned}
        \frac{(dz_j-dz_y)^p}{2^{p-1}}=g(\frac{dz_j-dz_y}{2})\leq g(-dz_y)\leq||f(\adv)-f(\xx)||_p^p,
    \end{aligned}
    \end{equation}
    which implies
    \begin{equation}
    \begin{aligned}
    f_j(\adv)-f_y(\adv)+f_y(\xx)-f_j(\xx)&\leq\frac{2}{2^{\frac{1}{p}}}||f(\adv)-f(\xx)||\\ &\leq cL||\adv-\xx||,
    \end{aligned}
    \end{equation}
    where the last inequality follows from the L-Lipschitz continuity of $f$.
    Moreover, since $f_j(\adv)-f_y(\adv)+f_y(\xx)-f_j(\xx)\geq q(\adv)+q(\xx)$, we have
    \begin{equation*}
        q(\adv) \leq cL||\adv-\xx||-q(\xx).
    \end{equation*}
    
\end{proof}

The proof for Theorem~\ref{theorem.lowerbound_of_distance} follows from a \emph{single} application of Lemma~\ref{lemma.certified_detection} on the pair $(\xx,\lcadv)$ together with the triangle inequality, requiring only $L \ge L^*_{(\lcadv,\xx)}$.

\begin{proof}
By construction the local centroid $\lcadv$ retains the predicted class of $\adv$, i.e., $y(\lcadv)=y(\adv)$. Combined with the assumption that $\adv$ is a successful adversarial example ($y(\adv)\neq y(\xx)$), the pair $(\xx,\lcadv)$ has \emph{distinct} predicted classes. Lemma~\ref{lemma.certified_detection} applied to this pair therefore yields
\begin{equation}\label{eq:lemma1_xclean_lcadv}
q(\lcadv) + q(\xx) \;\le\; cL\,\|\lcadv-\xx\|.
\end{equation}
By the triangle inequality and $\|\adv-\xx\|=\|\delta\|$,
\begin{equation}\label{eq:triangle_xclean_lcadv}
\|\lcadv-\xx\| \;\le\; \|\lcadv-\adv\| + \|\adv-\xx\| \;=\; \|\lcadv-\adv\| + \|\delta\|.
\end{equation}
Combining \eqref{eq:lemma1_xclean_lcadv} and \eqref{eq:triangle_xclean_lcadv} and dividing both sides by $cL>0$,
\begin{equation*}
\|\lcadv-\adv\| \;\ge\; \frac{q(\lcadv)+q(\xx)}{cL} - \|\delta\|,
\end{equation*}
which is the bound of Theorem~\ref{theorem.lowerbound_of_distance}. The inequality $L \geq L^*_{(\lcadv,\xx)}$ is exactly the saturation condition for \eqref{eq:lemma1_xclean_lcadv} to hold, and is the only Lipschitz requirement invoked.
\end{proof}

\subsection{Proof of Proposition~\ref{prop:approx_certifies}}\label{appendix:approx_certifies_proof}

\begin{proof}
For an adversarial input $z=\adv$, the pseudo-label is $\hat y=y(\adv)\neq y(\xx)$, and $q(\xx)>0$ at $\xx$'s own label.

\emph{(i) $\detector_{Adv}$.} Lemma~\ref{lemma.certified_detection} on the $(\xx,\adv)$ pair at $L=L^*_{(\xx,\adv)}$ gives $q(\adv)+q(\xx)\le c\,L^*_{(\xx,\adv)}\|\delta\|$. Combining with $L^*_{(\xx,\adv)}\le L'$ and the local-Lipschitz formula $c\,L'\!\cdot\!\max(\|\adv-z'\|,\|\delta\|)=q(\adv)-q(z';\hat y)$,
\[
\begin{aligned}
    q(\adv)+q(\xx)\;\le\;c\,L'\,\|\delta\|\;\le\;c\,L'\!\cdot\!\max(\|\adv-z'\|,\|\delta\|)\;\\
    =\;q(\adv)-q(z';\hat y),
\end{aligned}
\]
so $q(z';\hat y)\le -q(\xx)\le 0$ --- the flip property is a \emph{consequence} of $L'\ge L^*_{(\xx,\adv)}$, not a hypothesis. Substituting back,
\[
q(\adv)+q(z';\hat y)\;\le\;q(\adv)-q(\xx)\;\le\;q(\adv)+q(\xx)\;\le\;c\,L'\,\|\delta\|,
\]
which is the Claim~\ref{claim:approximate_adv} bound.

\emph{(ii) $\detector_{lc}$.} By Proposition~\ref{prop:pairwise_suffices}, $L=L^*_{(\lcadv,\xx)}$ certifies Theorem~\ref{theorem.lowerbound_of_distance}: $\|\adv-\lcadv\|\ge \tfrac{q(\lcadv)+q(\xx)}{c\,L^*_{(\lcadv,\xx)}}-\|\delta\|$. The detector's right-hand side is $\tfrac{q(\lcadv)+q(z';\hat y)}{c\,L'}-\|\delta\|$. Combining the flip property $q(z';\hat y)\le 0\le q(\xx)$ with $L^*_{(\lcadv,\xx)}\le L'$,
\[
\frac{q(\lcadv)+q(z';\hat y)}{c\,L'}\;\le\;\frac{q(\lcadv)}{c\,L'}\;\le\;\frac{q(\lcadv)}{c\,L^*_{(\lcadv,\xx)}}\;\le\;\frac{q(\lcadv)+q(\xx)}{c\,L^*_{(\lcadv,\xx)}},
\]
so the detector's right-hand side is at most the certified one, and the detector flags $\adv$ whenever Theorem~\ref{theorem.lowerbound_of_distance} does.
\end{proof}

\subsection{Margin Monotonicity Under Gradient Ascent}\label{appendix:margin_monotonicity}

We give a formal justification for the class-preservation property used by Theorem~\ref{theorem.lowerbound_of_distance}: under gradient ascent on the margin function at a fixed reference label, the margin is monotonically non-decreasing along the flow, hence the predicted class is preserved.

\begin{definition}[Chain rule]\label{def:chain_rule}
A locally Lipschitz function $h:\mathbb{R}^d \to \mathbb{R}$ admits a chain rule if for any absolutely continuous $u:\mathbb{R}_+\to\mathbb{R}^d$,
\[
(h\circ u)'(t) = \langle \partial^\circ h(u(t)),\, u'(t)\rangle \quad \text{for a.e. } t\ge 0.
\]
\end{definition}

\begin{lemma}[Davis et al~\cite{davis2020stochastic}, Lemma 5.2]\label{lemma:chain2}
Let $h:\mathbb{R}^d \to \mathbb{R}$ be locally Lipschitz and admit a chain rule, and let $u:\mathbb{R}_+\to\mathbb{R}^d$ be absolutely continuous with $du/dt \in \partial^\circ h(u(t))$. Then $\|du/dt\|^2 = dh(u(t))/dt$ for a.e.\ $t\ge 0$.
\end{lemma}

\begin{lemma}[Margin monotonicity]\label{lemma:margin_increasing}
Let $f$ be a locally Lipschitz network and $q(\cdot;\hat y)$ the margin function at fixed reference label $\hat y=\arg\max f(z)$ for input $z$, and assume $q(\cdot;\hat y)$ admits a chain rule (Definition~\ref{def:chain_rule}); this holds whenever $f$ is locally Lipschitz and tame, including piecewise-linear ReLU/max-pool networks~\cite{davis2020stochastic}. Under the gradient ascent flow $du/dt \in \partial^\circ q(u(t);\hat y)$ initialized at $u(0)=z$,
\[
\frac{d\,q(u(t);\hat y)}{dt} = \left\|\tfrac{du}{dt}\right\|^2 \;\ge\; 0 \quad \text{for a.e. } t\ge 0,
\]
with equality iff $du/dt=0$. In particular, $q(u(t);\hat y)$ is non-decreasing along the flow; if $q(z;\hat y)>0$, then $q(u(t);\hat y)>0$ for all $t\ge 0$, so $\arg\max f(u(t))=\hat y$ throughout.
\end{lemma}

\begin{proof}
The equality is Lemma~\ref{lemma:chain2} applied to $h=q(\cdot;\hat y)$. Non-negativity gives monotonicity. The class-preservation conclusion is $q(u;\hat y)>0 \iff \hat y=\arg\max f(u)$, combined with the trajectory's non-decreasing margin keeping the strict positivity from initialization onward.
\end{proof}

\subsection{Proof of Lemma~\ref{lemma:bias_analysis}}
\label{appendix:bias_proof}

\begin{proof}
Let $a=\adv=z$ and $b=\xx$ denote the adversarial and its (unknown) clean counterparts, respectively, 
and let $z'=z+\sigma$ be the perturbed sample used in our approximation.
Define the pairwise distances
\[
\alpha=\|a-b\|, \qquad \beta=\|a-z'\|, \qquad \gamma=\|b-z'\|.
\]
We write the \emph{ideal} constant from Lemma~\ref{lemma.certified_detection} as
\[
L^{\star} \;\triangleq\; \frac{q(a)+q(b)}{c\,\alpha},
\]
and the \emph{computable} local quotient as
\[
L' \;\triangleq\; \frac{|q(a)-q(z')|}{c\,\beta}.
\]

Let $L_{\mathrm{glob}}$ be any global Lipschitz constant of $q(\cdot)$, i.e.,
\[
|q(u)-q(v)| \le c\,L_{\mathrm{glob}}\,\|u-v\| \quad \forall u,v.
\]
Using the triangle inequality, we decompose:
\[
|q(a)-q(z')|
\le |q(a)-q(b)| + |q(b)-q(z')|.
\]
Since $|q(a)-q(b)| \le q(a)+q(b) = c\,\alpha L^{\star}$ 
and $|q(b)-q(z')| \le c\,L_{\mathrm{glob}}\gamma$, 
we have
\[
L' \le \frac{c\,\alpha L^{\star} + c\,L_{\mathrm{glob}}\gamma}{c\,\beta}
= \frac{\alpha}{\beta}\,L^{\star} + \frac{\gamma}{\beta}\,L_{\mathrm{glob}}.
\]

Similarly, by reversing the inequality,
\[
|q(a)-q(z')| 
\ge |q(a)-q(b)| - |q(b)-q(z')|
\ge c\,\alpha D - c\,L_{\mathrm{glob}}\gamma,
\]
where $D \triangleq \frac{|q(a)-q(b)|}{c\,\alpha} \le L^{\star}$.
Hence,
\[
L' \ge \frac{\alpha}{\beta}\,D - \frac{\gamma}{\beta}\,L_{\mathrm{glob}}.
\]
Combining both sides yields the two–sided bound:
\[
\frac{\alpha}{\beta}\,D - \frac{\gamma}{\beta}\,L_{\mathrm{glob}}
\;\;\le\;\;
L'
\;\;\le\;\;
\frac{\alpha}{\beta}\,L^{\star} + \frac{\gamma}{\beta}\,L_{\mathrm{glob}}.
\]

Consequently, the deviation of $L'$ from the scaled ideal constant $(\alpha/\beta)L^{\star}$ is bounded by
\[
\big|\,L' - (\alpha/\beta)\,L^{\star}\,\big|
\le \frac{\gamma}{\beta}\,L_{\mathrm{glob}}
     + \frac{\alpha}{\beta}\,(L^{\star}-D).
\]
This reveals two distinct sources of bias:
\begin{itemize}
    \item A \emph{geometric bias} $\tfrac{\gamma}{\beta}L_{\mathrm{glob}}$, governed by the relative placement of $z'$ with respect to $a$ and $b$;
    \item A \emph{margin bias} $\tfrac{\alpha}{\beta}(L^{\star}-D)$, reflecting the gap between the sum- and difference-based margin quotients.
\end{itemize}
In particular, when $z'$ lies approximately along the segment $(a,b)$, 
we have $\beta \approx \alpha$ and $\gamma \approx |\alpha-\beta|$, 
so both terms become small and $L'$ closely tracks $L^{\star}$. 
Conversely, when $z'$ is very close to $a$ (small $\beta$) or far from $b$ (large $\gamma$), 
the approximation tends to overestimate the ideal constant.
\end{proof}

\subsection{Proof of Theorem~\ref{thm:deterministic}}
\label{appendix:thm_determ_proof}

The proof of Theorem~\ref{thm:deterministic} is as follows.

\begin{proof}
We prove the theorem by bounding the difference term by term. We use the assumptions of Theorem~\ref{thm:deterministic}: $L_q\le c\bar L$ (\textbf{A1}) and $L'\le L_{\text{glob}}\le \bar L$ (\textbf{A2}); recall $\eta := 1 - L'/\bar L \in [0,1)$.

Let
\[
D = \left| \frac{q(\Tilde{z}) + q(\xx)}{c\,L_{\text{glob}}} - \frac{q(\Tilde{z}) + q(z';\hat y)}{c\,L'} \right| .
\]
By adding and subtracting $\tfrac{q(\Tilde{z}) + q(z';\hat y)}{c\,L_{\text{glob}}}$,
\[
D \leq \frac{1}{c} \left( \frac{|q(\xx) - q(z';\hat y)|}{L_{\text{glob}}} + |q(\Tilde{z}) + q(z';\hat y)| \cdot \left| \frac{1}{L_{\text{glob}}} - \frac{1}{L'} \right| \right) .
\]
By Lipschitz continuity of $q$ (\textbf{A1}) and the fact that $\|\xx - z'\|\le \|\xx-z\|+\|z-z'\|\le 2\|\delta\|$ (using $\|z-\xx\|\le\|\delta\|$ from the threat model and $\|z-z'\|\le\|\delta\|$ from Algorithm~\ref{alg:perturbation}), we have $|q(\xx) - q(z';\hat y)| \le L_q\|\xx-z'\|\le 2L_q\|\delta\|$. Combined with $L'\le L_{\text{glob}}$,
\[
\frac{|q(\xx) - q(z';\hat y)|}{L_{\text{glob}}} \;\le\; \frac{2L_q\|\delta\|}{L_{\text{glob}}} \;\le\; \frac{2L_q\|\delta\|}{L'} .
\]
For the second term, $L'\le L_{\text{glob}}\le \bar L$ gives $1 - L'/L_{\text{glob}}\ge 0$ and $1 - L'/L_{\text{glob}}\le 1 - L'/\bar L = \eta$. Therefore
\[
\left| \frac{1}{L_{\text{glob}}} - \frac{1}{L'} \right| \;=\; \frac{1}{L'}\cdot\frac{L_{\text{glob}} - L'}{L_{\text{glob}}} \;\le\; \frac{\eta}{L'} .
\]
Putting the two bounds together,
$$
\Delta_T \;\le\; \frac{1}{c}\left[\;\frac{2L_q}{L'}\|\delta\| \;+\;\frac{\eta}{L'}\,\big(q(\Tilde{z}) + |q(z';\hat y)| + L_q\|\delta\|\big)\;\right],
$$
which is the bound \eqref{eq:det-eta} (with the geometric stretch factor doubled to absorb the $\xx \to z \to z'$ triangle).
\end{proof}

\fi

% \end{document}

\end{document}